\documentclass[12pt,a4paper,reqno]{scrartcl}
\usepackage[utf8]{inputenc}

\usepackage{hyperref}
\usepackage{geometry}
\usepackage{enumerate} 
\usepackage{amssymb}
\usepackage{amsmath}
\usepackage{amsthm}
\usepackage[utf8]{inputenc}
\usepackage[english]{babel}
\usepackage{xcolor}
\usepackage{graphicx}

\usepackage{authblk}

\usepackage{pgfplots}
\pgfplotsset{compat=1.11}

\numberwithin{equation}{section}

\makeatletter
\newcommand{\customlabel}[2]{%
   \protected@write \@auxout {}{\string \newlabel {#1}{{#2}{\thepage}{#2}{#1}{}} }%
   \hypertarget{#1}{#2}
}
\makeatother

\newcommand\numberthis{\addtocounter{equation}{1}\tag{\theequation}}

\theoremstyle{definition}
\newtheorem{definition}{Definition}[section]

\theoremstyle{plain}
\newtheorem{remark}{Remark}[section]
\newtheorem*{remark*}{Remark}
\newtheorem{theorem}{Theorem}[section]

\newtheorem{lemma}{Lemma}[section]

\newtheorem{proposition}{Proposition}[section]

\title{Approximations with deep neural networks in Sobolev time-space}
\author[$*$]{Ahmed Abdeljawad}
\author[$*$, $\dagger$]{Philipp Grohs}

\affil[$\,$]{\footnotesize \texttt{ahmed.abdeljawad@ricam.oeaw.ac.at, \; philipp.grohs@univie.ac.at}
\vspace*{0.3cm}}

\affil[$*$]{\footnotesize
Johann Radon Institute for Computational and Applied Mathematics,\newline
Austrian Academy of Sciences,
Linz, Austria.
\vspace*{0.1cm}}

\affil[$\dagger$]{\footnotesize
Faculty of Mathematics,
University of Vienna, Vienna, Austria.
\vspace*{0.1cm}}

\date{\vspace{-5ex}}

\begin{document}

\maketitle

\begin{abstract}
Solutions of evolution equation generally lies in certain
Bochner-Sobolev spaces, in which the solution may has
regularity and integrability properties for the time variable that can be different for the space variables.
Therefore, in this paper, we develop a framework shows that deep neural networks can approximate Sobolev-regular functions with respect to Bochner-Sobolev spaces.
In our work we use the so-called Rectified Cubic Unit (ReCU) as an activation
function in our networks, which allows us to deduce approximation 
results of the  neural networks while avoiding issues caused by the non regularity of the most commonly used Rectivied Linear Unit (ReLU) activation function.
\end{abstract}


\section{Introduction}

\par

In recent years, methods from deep learning have been applied to the numerical solution of partial differential equations with impressive results
\cite{beck2018solving, Berg2018, Berner2020, Chen2019,elbrachter2018dnn, Gonon2019,
Grohs2020, Grohs2019, GrohsPEB2019, Hana2018, Hutzenthaler2020, Kutyniok2019, Lye2019,
Magill2018, sirignano2018dgm}.
One key component
 of this success lies in the expressive power of neural networks, which constitute a parametrizes class of functions constructed by iterative compositions of affine mappings and pointwise application of a nonlinear \emph{activation function}. Neural networks have been demonstrated to be at least on par with most known approximation methods, including (hp) finite elements, wavelts, shearlets, in terms of their approximation power,
 see for example \cite{boelcskei2019, chui2018, liang2017, mhaskar1993, opschoor2019deep, rolnick2018, shaham2018}.
 In these works it is shown that functions belonging to certain smoothness classes can be approximated by neural networks at a complexity corresponding to the optimal approximation rate as dictated by the metric entropy of the smoothness class, where the approximation error is typically measured in an isotropic Sobolev norm. However, if one considers the problem of approximating solutions to time dependent partial differential equations, the natural norm in which 
 the error is measured are typically of a different form coming from a space time Sobolev space
 \cite{arendt2014, cazenave2010,fefferman2017,kreiss2004,rothe2006}.
 Motivated by this fact we consider in this paper the approximation of functions in space time Sobolev spaces. Our main result Theorem \ref{thm:main} shows that, similar to functions in isotropic Sobolev spaces, these functions can be efficiently approximated by neural networks, also when the approximation error is measured with respect to a space time Sobolev norm. Our result is constructive and similar in spirit to \cite{GuhKutPet} in the sense that our explicit constructions of approximants emulate a local polynomial approximant. In contrast to
 \cite{masterthesis, GuhKutPet, petersen2017optimal, yarotsky2017error}
 where the so-called ReLU function $x\mapsto \max\{0,x\}$ is chosen as activation function, our results hold for the so-called ReCU function $x\mapsto \max\{0,x\}^3$, the main reason being that the latter is continuously differentiable. 
 
 \subsection{Outline}

This paper is organized as follows. 
In Section \ref{sec:prelim}
we provide definitions and properties of
the Sobolev time-space. All the proofs of the result in Section \ref{sec:prelim} can be found in Appendix A.
We start Section \ref{sec:network} by introducing the mathematical definition of deep neural networks, moreover we show some of its properties. 
Finally, in Section \ref{sec:main}, we prove the main result of our paper.
That is, Theorem \ref{thm:main}, where we show that deep neural networks
can approximate certain function in Sobolev time-space, thus
we get information about regularity and approximation rate.
For the approximation and estimation results, we always work with ReCU activation function.

\subsection{Notations}

\par
Throughout the paper, the following notation is used: The sets of natural numbers and
real numbers are denoted by $\mathbb{N}$ and $\mathbb{R}$, respectively.
Furthermore,  $\mathbb{N}_{0}=\{0\}\cup \mathbb{N}$, denotes the set of non-negative integers.

If $x \in \mathbb{R}$, then we write $\lceil x\rceil:=\min \{k \in \mathbb{Z}: k \geq x\}$ where $\mathbb{Z}$ is the set of integers.

If $d \in \mathbb{N}$ and $\|\cdot\|$ is a norm on $\mathbb{R}^{d},$ then we denote for
$x \in \mathbb{R}^{d}$ and $r>0$ by $B_{r,\|-\|}(x)$ the open ball around $x$
in $\mathbb{R}^{d}$ with radius $r,$ where the distance is measured in $\|\cdot\|$.
By $|x|$ we denote the Euclidean norm of $x$ and by $|x|_{\ell \infty}$ the maximum norm.
Moreover, throughout this paper $\| \cdot  \|_{\ell^0}$  be referred to as
the counting norm that return the total number of non-zero elements in a given vector.
Strictly speaking, $\ell^0$-norm is not actually a norm in the mathematical sense.

We endow $\mathbb{R}^{d}$ with the standard topology and for
$A \subset \mathbb{R}^{d}$ we denote by $\bar{A}$ the closure of $A$
and by $\partial A$ the boundary of $A$.
The diameter of a non-empty set $A \subset \mathbb{R}^{d}$ is always taken
with respect to the euclidean distance, i.e.
$\operatorname{diam} A:=\operatorname{diam}_{|\cdot|} A:=\sup _{x, y \in A}|x-y|$.
If $A, B \subset \mathbb{R}^{d},$ then we write $A \subset \subset B$
if $\bar{A}$ is compact in $\bar{B}$.
Let $d, n\in \mathbb{N}$, $\Omega \subset \mathbb{R}^{d}$ be open,
then $C^{n}(\Omega)$ stands for the set of $n$ times continuously differentiable functions on $\Omega$
and $C^{\infty} =\bigcap_{k=1}^{\infty} C^{k}(\Omega)$.

Note that if $\alpha =\left(\alpha_{1}, \alpha_{2}, \ldots, \alpha_{d}\right)\in \mathbb{N}^d$
is a multi-index, then $\alpha! = \alpha_1!\dots\alpha_d!$
and $|\alpha|=\alpha_{1}+\ldots+\alpha_{d}$.
Let $x = (x_1, \dots, x_d)\in \mathbb{R}^d$, 
and $u \in C^{n}(\Omega)$, and let
$\alpha \in \mathbb{N}^{d}$ 
be a multi-index such that $|\alpha| \leq n$, then we denote 
$$
D^{\alpha} u=\frac{\partial^{|\alpha|} u}{\partial x_{1}^{\alpha_{1}} \ldots \partial x_{d}^{\alpha_{d}}}=\frac{\partial^{\alpha_{1}}}{\partial x_{1}^{\alpha_{1}}} \cdots \frac{\partial^{\alpha_{d}}}{\partial x_{d}^{\alpha_{d}}} u,\quad D_{x_j}u =  \frac{\partial}{\partial{x_j}}u, \text{ where }1\leq j\leq d.
$$
Let $\Omega\subset \mathbb{R}^d$, $1 \leq p \leq \infty$,
then $L^p(\Omega, \mathbb{R})=L^p(\Omega)$ denotes the Lebesgue space.
Moreover, if $X$ is a Banach space,
then the space $L^{p}(\Omega, X)$ is called vector-valued Lebesgue space or Bochner space, 
defined as the space of all measurable functions such that $\|f\|_X \in L^{p}(\Omega, \mathbb{R})$ and the norm on this space will be defined via $\|f\|_{L^{p}(\Omega, X)}:=\|\| f\|_X\|_{L^{p}(\Omega, \mathbb{R})}$.

If $d_{1}, d_{2}, d_{3} \in \mathbb{N}$ and $A \in \mathbb{R}^{d_{1}, d_{2}}, B \in \mathbb{R}^{d_{1}, d_{3}},$ then we use the common block matrix notation and write for the horizontal concatenation of $A$ and $B$

$$
\left[\begin{array}{lll}
A & | &B
\end{array}\right]
\in \mathbb{R}^{d_{1}, d_{2}+d_{3}}.
$$
A similar notation is used for the vertical concatenation of
$A \in \mathbb{R}^{d_{1}, d_{2}}$ and $B \in \mathbb{R}^{d_{3}, d_{2}}$.

We define the Rectified Power Unit (RePU) as follows
\begin{equation} \label{eq:ReQU}
\rho_s(x) =
\begin{cases}
x^s, & x \ge 0,
\\ 0, & x<0,
\end{cases}, s \in \mathbb{N}_0.
\end{equation}

Note that
 $\rho_{0}$ is the binary step function while $\rho_{1}$
 is the commonly used Rectified Linear Unit (ReLU) function.
We call $\rho_{2}, \rho_{3}$ Rectified Quadratic Unit $(\mathrm{ReQU})$
and Rectified Cubic Unit (ReCU), respectively.
\par


\section{Sobolev time-space definition and properties}\label{sec:prelim}

\par

In the current section we review some properties of mixed Sobolev spaces
and extend some. Moreover we show that Bramble-Hilbert lemma is valid in our setting.
More details about the proofs can be found in the Appendix.

\begin{definition}[Sobolev space]
Assume that $\Omega$ is an open subset of $\mathbb{R}^{d}$, and
let $n \in \mathbb{N}$, $1 \leq p \leq \infty$. 
The Sobolev space $W^{n, p}(\Omega)$ consists of functions $u \in L^{p}(\Omega)$ such that for every multi-index $\alpha$ with $|\alpha| \leqslant k$,
$D^{\alpha} u$ exists and $D^{\alpha} u \in L^{p}(\Omega)$. Thus
\[
W^{n, p}(\Omega):=\left\{f \in L^{p}(\Omega): D^{\alpha} f \in L^{p}(\Omega) \text { for all } \alpha \in \mathbb{N}_{0}^{d} \text { with }|\alpha| \leq n\right\}.
\]
Furthermore, for $f \in W^{n, p}(\Omega)$ and $1 \leq p<\infty,$ we define the norm
\[
\|f\|_{W^{n, p}(\Omega)}:=\left(\sum_{0 \leq|\alpha| \leq n}\left\|D^{\alpha} f\right\|_{L^{p}(\Omega)}^{p}\right)^{1 / p}
\]
and
\[
\|f\|_{W^{n, \infty}(\Omega)}:=\max _{0 \leq|\alpha| \leq n}\left\|D^{\alpha} f\right\|_{L^{\infty}(\Omega)}.
\]
\end{definition}
\begin{definition}[Sobolev time-space]
Let $1 \leq p,q \leq \infty$, $m, n\in \mathbb{N}$,
$I \subset\subset \mathbb{R}$,
and  $\Omega\subset\subset \mathbb{R}^d$. Let
$W_{m,q}^{n, p}(I, \Omega)$
defined as follows 
$$
W_{m,q}^{n, p}(I, \Omega)=\left\{f \in L^{q}
\left(I, W^{n, p}(\Omega)\right): \partial _t^ k f \in L^{q}\left(I, W^{n, p}(\Omega)\right)
 \text { for all }  k \leq m\right\}
$$
such that
$$
\Vert f\Vert _{W_{m,q}^{n, p}(I, \Omega)}=
\sum_{k \leq m}\Vert \partial _t^ k f\Vert_{L^{q}\left(I, W^{n, p}(\Omega)\right)}
$$
when $1\leq p, q< \infty$,
with the obvious modifications when $p=\infty$ and/or $q=\infty$.
\end{definition}

\par

Note that if $n=m=0$, then $W_{0,q}^{0,p}(I, \Omega) = L^q(I, L^p(\Omega))$.
Hence, we shall write $L_t^qL_x^p(I\times\Omega) := W_{0,q}^{0,p}(I, \Omega)$, where $L_t^q$ and $L_x^p$
stand for the Legesgue integral with respect to $t\in I$ and $x\in \Omega$, respectively.

\par

Next we introduce the Sobolev time-space semi-norms in order to simplify
the notations in the proofs came in the sequel.

\begin{definition}[Sobolev time-space semi-norm]\label{def:sobolev_seminorms}
 $I \subset\subset\mathbb{R}$, $\Omega\subset\subset\mathbb{R}^d$ .
 For $n, k \in \mathbb{N}_{0}$ with $\ell \leq k$, $m \leq n, m \in \mathbb{N}$
and $1 \leq p, q \leq \infty$, 
we define for $f \in W_{k,q}^{n, p}\left(I,\Omega\right)$ the Sobolev time-space semi-norm
$$
|f|_{W_{\ell ,q}^{m, p}\left(I, \Omega\right)}:=
\left(\sum_{|\alpha|=m}\left\|D_x^{\alpha} D_t^\ell f\right\|_{L^qL^{p}(I\times\Omega)}^{p}\right)^{1 / p} \quad \text { for } 1 \leq p, q<\infty
$$
and
$$
\begin{array}{c}
|f|_{W_{\ell, q}^{m, \infty}\left(I, \Omega\right)}:=
\max _{|\alpha|=m}\left\|D_x^{\alpha} D_t^\ell f\right\|_{L_t^qL_x^{\infty}(I, \Omega)},
\end{array}
$$
with the obvious modification when $q=\infty, 1\leq p<\infty$ and when $p=q=\infty$.
\end{definition}
\begin{definition}\label{def:TaylorPolynom}
Let $I \subset\subset\mathbb{R}$, $\Omega\subset\subset\mathbb{R}^d$ and  $m\in \mathbb{N}$.
Then the Taylor polynomial of order $m$ evaluated at $(\tau, \xi)\in I\times \Omega$ is given
by
\begin{equation}\label{eq:sob_taulor_sub}
T^m _{\tau, \xi}u(t, x) = \sum_ {k+|\alpha|< m}
\frac 1{\alpha!k!}D _x^\alpha D_t ^k u(\tau, \xi)(x-\xi)^\alpha(t-\tau)^k,
\end{equation}

where $\alpha$ is the $d$-tuple of nonnegative integers and $k\in\mathbb{N}_0$.
\end{definition}

\par 

\begin{definition}[averaged Taylor polynomial]\label{def:taylor}
Let $I \subset\subset\mathbb{R}$, $\Omega\subset\subset\mathbb{R}^d$ ,
$k, n\in\mathbb{N}_0$, $m\in\mathbb{N}$,
such that $k + n \in \{0, \dots, m-1\}$,
$1\leq p, q\leq \infty$
and $u \in W_{k, q}^{n,p}(I\times\Omega)$, and let
$(t_0, x_0)\in I\times\Omega$, $r>0$ such that for the ball
$\mathrm{B}:=\{ (t,x)\in I\times \Omega \text{ such that } |t-t_0|+|x-x_0| <  r\}$
it holds that $\mathrm{B} \subset\subset \Omega$.
The corresponding \emph{Taylor polynomial of order $m$ of $u$ averaged over $\mathrm{B}$}
is defined for $(t,x)\in I\times \Omega$ as 
\begin{equation}\label{eq:sob_taylor}
Q^m u(t, x):=\int_\mathrm{B} T^m_{\tau, \xi} u(t,x)\phi(\tau, \xi)\, d\xi d\tau,
\end{equation}
where $ T^m_{\tau, \xi} u$ is the Taylor polynomial of order $m$
defined in Definition \ref{def:TaylorPolynom},
and $\phi$ is an arbitrary cut-off function supported in $\overline{\mathrm{B}}$,
\(
\text{ with }\phi(t, x)\geq 0
\text{ for all }(t, x)\in\mathbb{R}\times\mathbb{R}^d,
 \text{ supp~} \phi = \overline{\mathrm{B}} \text{ and }
 \int_{\mathbb{R}}\int_{\mathbb{R}^d}\phi(t, x)\, dxdt=1.
\)
\end{definition}

A cut-off function as used in the previous definition always exists. A possible choice is
\begin{equation*}
\phi(t, x)=\begin{cases}
e^{-\left(1-(\vert{t-t_0}\vert/r)^2\right)^{-1}-\left(1-(\vert{x-x_0}\vert/r)^2\right)^{-1}},
& \text{if }\vert{t-t_0}\vert+ \vert{x-x_0}\vert<r\\
0,&\text{else}
\end{cases}
\end{equation*}
normalized by $\int_{\mathbb{R}}\int_{\mathbb{R}^d}\phi(t, x) dxdt$.
		
\par

\begin{proposition}\label{prop:remainder_estimate}
Let $C_{t,x}$ denotes the convex hull of $\{(t, x)\} \cup  \mathrm{B}$.
Then, the remainder $R^m u := u - Q^m u$ satisfies
$$
R^m u(t, x) = m \sum _{|\alpha|+ k=m}\int _{C_{t,x}} K_{ \alpha, k}(t, T; x, \Xi)
 D_x^\alpha D_t^k u(T, \Xi) \, d\Xi dT
$$
where $\Xi = x+ s(\xi-x), T  = t + s(\tau-t)$,
$K_{\alpha, k}(t, T; x, \Xi)= \frac{1}{\alpha!k!} ( x-\Xi)^{\alpha} (t-T)^k K(t, T; x, \Xi)$
and
\begin{equation}\label{eq:kernelEstim}
\left|K(t, T; x, \Xi) \right|\leq C\left(1 + \left(\left|x-x_{0}\right|+ |t- t_0|\right)/r\right)^{d+1}
			\left(|\Xi-x|+ |T- t| \right)^{-d-1}.
\end{equation}
\end{proposition}

\begin{proof}
The proof can be found in Appendix \ref{app:prop_remainder_estimate}.
\end{proof}
\par

Next we recall some geometric definitions needed for the control of the non-degeneracy of
a given family of subdivisions of a domain $\Omega$ through the so-called
chunkiness parameter. More details can be found in the discussion after
\cite[Definition 10.5.1]{brenner2007mathematical}.

\par

\begin{definition}
Let $\Omega, \mathrm{B}\subset\subset\mathbb{R}^d$, then $\Omega$ is star-shaped with respect to
$\mathrm{B}$ if, for all $x \in \Omega,$ the closed convex hull
of $\{x\} \cup \mathrm{B}$ is a subset of $\Omega$.
\end{definition}

\par

\begin{definition}
Let  $\Omega\subset \subset\mathbb{R}^d$ have diameter $diam(\Omega)>0$ and is star-shaped
with respect to a ball $\mathrm{B}$.
Let $\mathcal{R}= \{r>0: \Omega$ is star-shaped with respect to $a$ ball of radius $r\}$.
If $\mathcal{R}\neq \emptyset$, then $r_{\max }^\star = \sup\mathcal{R}$ and
 the chunkiness parameter of $\Omega$ is defined by
$$
\gamma=\frac{diam(\Omega)}{r_{\max }^\star}
$$
\end{definition}
\par

\begin{lemma}[Bramble-Hilbert]\label{lemma:bramble_hilbert}
Let $I\subset\subset\mathbb{R}$, $\Omega\subset\subset \mathbb{R}^d$,
$t_0\in I$, $x_0\in\Omega$ and $r>0$
such that $I\times\Omega$ is star-shaped with respect to $\mathrm{B}:=B_{r}(t_0, x_0)$,
and $r>(1/2)r^{\star}_{\max}$. Moreover, let $k, m, n\in\mathbb{N}$, 
such that  $k + n\in \{0,1,\ldots, m\}$, $1\leq p, q\leq\infty$ and denote
by $\gamma$ the chunkiness parameter of $I\times\Omega$.
Then there exists a constant
$C=C(m,d,\gamma)>0$ such that for all $u\in W_{m, q}^{m, p}(I, \Omega)$
\begin{equation*}
\big\vert{u-Q^m u}\big\vert_{W_{k, q}^{n,p}(I, \Omega)} \leq C h^{m-k- n}
	\big\Vert{u}\big\Vert_{W_{m, p}^{m,p }(I, \Omega)},
\end{equation*}
where $Q^m u$ denotes the Taylor polynomial of order $m$ of $u$ averaged
over $\mathrm{B}$ and $h=diam(I\times \Omega)$.
\end{lemma}

\begin{proof}
A proof can be found in Appendix \ref{app:lemma_bramble_hilbert}.
\end{proof}
\par

\begin{lemma}\label{prop:taylor_is_polynom}
Let $I\subset\subset\mathbb{R}$, $\Omega\subset\subset \mathbb{R}^d$,
$k, n\in\mathbb{N}$, $1\leq p\leq \infty$ and $u\in W_{k +n-1,p}^{k+n-1, p}( I,\Omega)$,
and let $(t_ 0, x_0)\in I\times\Omega$,
$r>0,R\geq 1$ such that for the ball $\mathrm{B}:=B_{r,|{\cdot}|}((t_0,x _0))$
it holds that $\mathrm{B} \subset\subset  I\times\Omega$ and $\mathrm{B}\subset B_{R,\Vert{\cdot}\Vert_{\ell^\infty}}(0)$.
Then the Taylor polynomial of order $n+k$ of $u$ averaged over $\mathrm{B}$ can be written as
\[
Q^{n +k} u(t, x)=\sum_{|{\alpha}|+ \kappa\leq k+ n-1} c_{\alpha, \kappa}t ^\kappa x^\alpha
\]
for $(t, x) \in I\times\Omega$. 
		
Moreover, there exists a constant $c=c(k, n,d,R)>0$ such that the coefficients $c_{\alpha, \kappa}$
are bounded with $|{c_{\alpha, \kappa}}|\leq c r^{-(d+1)/p}
\Vert{u}\Vert_{W_{k+n -1, p}^{k+ n-1, p}( I, \Omega)}$
for all $\alpha, \kappa$ with $|{\alpha}|+\kappa\leq k+ n-1$.
\end{lemma}

\begin{proof}
A detailed proof can be found in Appendix \ref{app:prop_taylor_is_polynom}.
\end{proof}
\par

We need the following lemma to estimate the semi-norm of a product of
weakly differentiable functions on the mixed Sobolev space.

\par

\begin{lemma}\label{lemma:product_rule_bound_p}
Let $1\leq p, q\leq \infty$, and  $I\subset\subset\mathbb{R}$, $\Omega\subset\subset \mathbb{R}^d$,
 $f\in W_{1, \infty}^{1, \infty}(I, \Omega)$, and $g\in W_{1, q}^{1, p}(I, \Omega)$,
 then $fg\in W_{1, q}^{1, p}(I, \Omega)$
and there exists a constant $C_1, C_2>0$ depend on $d$ and $p$ such that
\begin{align*}
|{fg}|_{W_{0,q}^{1,  p}((I, \Omega))}&\leq  C_1 \left(|f|_{W_{0, \infty}^{1,\infty}(I,\Omega)}
\Vert{g}\Vert_{L_t^q L_x^p(I\times\Omega)}
+ \Vert{f}\Vert_{L _t^\infty L_x^\infty}\vert{g}\vert_{W_{0,q}^{1, p}(I,\Omega)}\right),
\\[1ex]
|{fg}|_{W_{1,q}^{0,  p}(I, \Omega)} &\leq
\left|{f}\right|_{W_{1,\infty}^{0, \infty}(I, \Omega)}
 \left\Vert{g}\right\Vert_{L_t^qL_x^p(I\times\Omega)}
+\left\Vert{f }\right\Vert_{L_t^\infty L_x^\infty(I\times\Omega)}
		\left|{g}\right|_{W_{1, q}^{0,p}(I, \Omega)},
\\[1ex]
|{fg}|_{W_{1,q}^{1,  p}(I, \Omega)}&\leq
C_2\left( |{f}|_{W_{1, \infty}^{1, \infty}}\Vert{g}\Vert_{L_t^qL_x^p(I\times\Omega)}
+|{f}|_{W_{1, \infty}^{0, \infty}}
\vert{g}\vert_{W_{0, q}^{1,p}(I,\Omega)}\right.
\\
&\qquad \qquad \left.+|{f}|_{W_{0, \infty}^{1, \infty}} |{g}|_{W_{1, q}^{0, p}(I,\Omega)}
+ \Vert{f}\Vert_{L_t^\infty L_x^\infty}
\vert{ g}\vert_{W_{1,q}^{1, p}(I,\Omega)}\right).
\end{align*}
For $p=\infty$, we have $C_1=C_2=1$.
\end{lemma}

\begin{proof}
A proof can be found in Appendix \ref{app:lemma_product_rule_bound_p}.
\end{proof}
\par

The following corollary establishes a chain rule estimate for $W_{1, \infty}^{1,\infty}$.
\begin{lemma}\label{cor:composition_norm}
	Let $p_i, m_i\in\mathbb{N}$, for $i=1,2$, $n, k\in \{0,1\}$ such that $n+k=1$,
	$p_1+ p_2=p$, and $ m_1+m_2=m$	and let
	$\Omega_i\subset \mathbb{R}^{p_i},\, \Theta_i\subset \mathbb{R}^{m_i}$
	with $i=1,2$,
	be open, bounded, and convex. Then, there is a constant $C=C(p,m)>0$
	with the following property:
	
	If $p_1=m_1=1$, $f\in W_{k, \infty}^{n,\infty}(\Omega_1,\Omega_2)$ and
	$g\in W_{k,\infty}^{n,\infty}(\Theta_1,\Theta_2)$
	are Lipschitz continuous functions such that $ran(f)\subset \Theta_1\times\Theta_2$,
	then
	$g\circ f\in W_{k, \infty}^{n,\infty}(\Omega_1,\Omega_2)$
	and we have 
	\[
		\vert g\circ f\vert_{W_{k, \infty}^{n,\infty}(\Omega_1,\Omega_2)}\leq C
		\vert g\vert_{W_{k, \infty}^{n,\infty}(\Theta_1, \Theta_2)}
		\vert f\vert_{ W_{k, \infty}^{n,\infty}(\Omega_1,\Omega_2)}.
	\]
Moreover, if $n=k=1$ there exists $C'=C'(p,m)>0$, such that
\[
	\vert g\circ f\vert_{W_{1, \infty}^{1,\infty}(\Omega_1,\Omega_2)}
	 \leq C'\max\left(\vert f\vert_{W_{1, \infty}^{1,\infty}(\Omega_1,\Omega_2)}
	 \vert g\vert_{W_{1, \infty}^{0, \infty}(\Theta_1,\Theta_2)},
	\vert f\vert_{W_{0, \infty}^{1,\infty}(\Omega_1,\Omega_2)}^2
	\vert{g}\vert_{W_{1, \infty}^{1,\infty}(\Theta_1, \Theta_2)}\right).
\]
\end{lemma}

\begin{proof}
The proof of  Lemma \ref{cor:composition_norm} can be found in Appendix \ref{app:cor_composition_norm}.
\end{proof}


\section{Mathematical theory of neural networks}\label{sec:network}
Deep  neural  networks  have  been  shown  to  perform  well  on  classification  or  regression  tasks,
that is supervised learning problems.

%
%
%
%

\par

Here we introduce the basic mathematical theory of neural networks
that will be used during this paper.

\par

\begin{definition}
Let $d, L \in \mathbb{N} .$ A neural network $\Phi$
with input dimension $d$ and $L$ layers is a sequence of matrix-vector tuples
$$
\Phi=\left(\left(A_{1}, b_{1}\right),\left(A_{2}, b_{2}\right),
\ldots,\left(A_{L}, b_{L}\right)\right)
$$
where $N_{0}=d$ and $N_{1}, \ldots, N_{L} \in \mathbb{N}$,
and where each $A_{\ell}$ is an $N_{\ell} \times N_{\ell-1}$ matrix,
and $b_{\ell} \in \mathbb{R}^{N_{\ell}}$.
If $\Phi$ is a neural network as above,
and if $\rho: \mathbb{R} \rightarrow \mathbb{R}$ is arbitrary,
then we define the associated realization of $\Phi$
with activation function $\rho$ as the map 
$\mathrm{R}_{\rho}(\Phi): \mathbb{R}^{d} \rightarrow \mathbb{R}^{N_{L}}$ such that
$${R}_{\rho}(\Phi)(x)=x_{L}$$
where $x_{L}$ results from the following scheme:
\begin{align*}
x_{0}:= & x \\
x_{\ell}:= &\rho\left(A_{\ell} x_{\ell-1}+b_{\ell}\right),
\quad \text{ for }\ell=1, \ldots, L-1 \\
x_{L}:= & A_{L} x_{L-1}+b_{L}
\end{align*}
where $\rho$ acts componentwise, i.e., for a given vector $y\in \mathbb{R}^{m}$,
$\rho(y)=\left[\rho\left(y_{1}\right), \ldots, \rho\left(y_{m}\right)\right]$.
\end{definition}

\par

We call $N(\Phi):=d+\sum_{j=1}^{L} N_{j}$ the number of neurons of the network $\Phi,$ while $L(\Phi):=L$ denotes the number of layers of
$\Phi .$ Moreover, $M(\Phi):=\sum_{j=1}^{L}\left(\left\|A_{j}\right\|_{\ell^{0}}+\left\|b_{j}\right\|_{\ell^{0}}\right)$ denotes the total
number of nonzero entries of all $A_{\ell}, b_{\ell},$ which we call the number of weights of $\Phi$. Finally, we refer to $N_{L}$ as the dimension of the output layer of $\Phi$, or simply as the output dimension of $\Phi$. We shall also sometimes refer to $\mathcal{A}(\Phi):= (N_0,\dots , N_L)\in \mathbb{N}^{L+1}$ as the architecture of $\Phi$.

\par

When dealing with neural networks, usually one has to fix
a specific architecture (see Definition \ref{def:architecture})
e.g., fully-connected feedforward neural networks
where  information in such architecture flows in one direction from input to output layer (via hidden nodes if any), that is
they do not form any circles or loopbacks.
More details about different architecture can be found in
e.g., \cite{KhaSohZahQue2020, ShrMah2019}.

\par

\begin{definition}\label{def:architecture}
Let $d, L \in \mathbb{N}$, a neural network architecture $\mathcal{A}$
with input dimension $d$ and $L$ layers is a sequence of matrix-vector tuples
$$
\mathcal{A}=\left(\left(A_{1}, b_{1}\right),\left(A_{2}, b_{2}\right), \ldots,\left(A_{L}, b_{L}\right)\right)
$$
such that $N_{0}=d$ and $N_{1}, \ldots, N_{L} \in \mathbb{N}$,
where each $A_{l}$ is an $N_{l} \times \sum_{k=0}^{l-1} N_{k}$ matrix,
and $b_{l}$ a vector of length $N_{l}$ with elements in $\{0,1\}$.
We call $N(\mathcal{A}):=d+\sum_{j=1}^{L} N_{j}$ the number of neurons
of the architecture $\mathcal{A}, L(\mathcal{A}) = L$ the number of layers and
$M(\mathcal{A}):=
\sum_{j=1}^{L}\left(\|A_{j}\|_{\ell^{0}}+\|b_{j}\|_{\ell^{0}}\right)$
Moreover, $N_{L}$  denotes the dimension of the output layer of $\mathcal{A}$.
We say that a neural network
$\Phi=\left((A_{1}^{\prime}, b_{1}^{\prime}),(A_{2}^{\prime}, b_{2}^{\prime}), \ldots,
(A_{L}^{\prime}, b_{L}^{\prime})\right)$
with input dimension $d$ and $L$ layers has architecture $\mathcal{A}$ if
the followings are satisfied
\begin{enumerate}[(i)]
\item $N_{l}(\Phi)=N_{l}$ for all $l=1, \ldots, L$,
\item $\left[A_{l}^{\prime}\right]_{i, j} \neq 0$  implies
    $\left[A_{l}\right]_{i, j} \neq 0$ such that $l=1, \ldots, L$
    where $ i=1, \ldots, N_{l}$ and $j=1, \ldots, \sum_{k=0}^{l-1} N_{k}$.
\end{enumerate}
\end{definition}

\par

Throughout the paper, we consider the Rectified Cubic Unit (ReCU)
activation function, which is defined as follows:
\begin{equation}\label{eq:recu}
\rho_3: \mathbb{R} \rightarrow \mathbb{R}, \quad x \mapsto \max (0, x^3).
\end{equation}

To construct new neural networks from existing ones, we will frequently need to concatenate networks or put them in parallel.
Most of the following results are well-known, see for example \cite{petersen2017optimal}.
We first define the concatenation of networks.
\begin{definition}\label{def:concat}
Let $L_{1}, L_{2} \in \mathbb{N}$, and let
$$\Phi^{1}=\left(\left(A_{1}^{1}, b_{1}^{1}\right), \ldots,\left(A_{l_{1}}^{1}, b_{l_{1}}^{1}\right)\right),
\quad \Phi^{2}=\left(\left(A_{1}^{2}, b_{1}^{2}\right), \ldots,\left(A_{l_{2}}^{2}, b_{l_{2}}^{2}\right)\right)
$$
be two neural networks such
that the input layer of $\Phi^{1}$ has the same dimension as the output layer of $\Phi^{2}$. Then, $\Phi^{1} \bullet \Phi^{2}$ denotes the following $L_{1}+L_{2}-1$ layer network:
\[
\begin{aligned}
\Phi^{1} \bullet \Phi^{2}:=&\left(\left(A_{1}^{2}, b_{1}^{2}\right), \ldots,\left(A_{L_{2}-1}^{2}, b_{L_{2}-1}^{2}\right),\left(A_{1}^{1} A_{L_{2}}^{2}, A_{1}^{1} b_{L_{2}}^{2}+b_{1}^{1}\right),
\left(A_{2}^{1}, b_{2}^{1}\right), \ldots,\left(A_{L_{1}}^{1}, b_{L_{1}}^{1}\right)\right).
\end{aligned}
\]
We call $\Phi^{1} \bullet \Phi^{2}$ the concatenation of $\Phi^{1}$ and $\Phi^{2}$.
\end{definition}

\par

\begin{lemma}
Let $\Phi^{1}$ and $\Phi^{2}$ be two neural networks
where the input layer of $\Phi^{1}$ has the same dimension as the output layer of  $\Phi^{2}$,
then
\begin{equation}\label{eq:concatination_id}
{R}_{\rho_3}\left(\Phi^{1} \bullet \Phi^{2}\right)={R}_{\rho_3}\left(\Phi^{1}\right) \circ {R}_{\rho_3}\left(\Phi^{2}\right).
\end{equation}
\end{lemma}

\begin{proof}
Equality in \eqref{eq:concatination_id} is immediate and follows from the previous
Definition \ref{def:concat}.
\end{proof}

Next we show that small neural networks are capable of emulating the identity.
\par

\begin{lemma}\label{lem:ReCU_id_representation}
Let $\rho_3$ be the ReCU,
$\Omega_r = \prod_{j=1}^{d}[-r_j, r_j]$, where $r_j>0$,
let $d \in \mathbb{N},$ and define two layers neural network
$\Phi_{d, r}^{\mathrm{ld}}:=\left(\left(A_{1}, b_{1}\right),\left(A_{2}, b_{2}\right)\right)$
with
$$
A_{1}:=\left(\begin{array}{c}
\mathrm{Id}_{\mathrm{R}^{d}} \\
-\mathrm{Id}_{\mathrm{R}^{d}} \\
\mathrm{Id}_{\mathrm{R}^{d}}\\
-\mathrm{Id}_{\mathrm{R}^{d}}
\end{array}\right) \quad
b_{1}:=\left(\begin{array}{c}
r_1+2\\
\vdots\\
r_d+2 \\
r_1,\\
\vdots\\
r_d \\
r_1\\
\vdots\\
r_d\\
r_1+2\\
\dots\\
r_d+2
\end{array}\right),$$
$$
A_{2}:=1/24\left(\begin{array}{c}
diag(1/(r_1+1), \dots, 1/(r_d+1)) \\
diag(1/(r_1+1), \dots, 1/(r_d+1)) \\
-diag(1/(r_1+1), \dots, 1/(r_d+1))\\
-diag(1/(r_1+1), \dots, 1/(r_d+1))
\end{array}\right)^t, \quad b_{2}:=0.
$$
Then, the realization  ${R}_{\rho_3}\left(\Phi_{d, r}^{\mathrm{ld}}\right)
= {Id}_{\Omega_r}$.
\end{lemma}

\par

\begin{proof}
The proof of the lemma follows from the following
identity 
$$
x = \frac{1}{24(r+1)}\Big(\rho_3(x+r+2) + \rho_3(-x +r)-\rho_3(x+r) - \rho_3(-x +r+2)\Big),
$$
for any $x \in [-r, r]$ where $r>0$. The extension to general domain is straightforward,
thus the details are left for the reader.
\end{proof}
\par

\begin{remark}\label{remark:neural_sparse_concat}
In view of Definition \ref{def:concat}, we can bound the number of layers, neurons and weights as follows
$$L(\Phi^{1} \bullet \Phi^{2})=L^1+L^2-1\leq L^1 + L^2,$$
$$N(\Phi^{1} \bullet \Phi^{2})=N^1+N^2 - N_0^2 - N_{L_1}^1 \leq N_1 + N_2,$$
$$M(\Phi^{1} \bullet \Phi^{2})\leq M^1 +M^2 + M^1M^2.$$
\end{remark}
\par

In the current paper we need another operation between networks, which is the parallelization.
That is, one can put two networks of same length in parallel as next definition shows.

\par

\begin{definition}\label{def:parallel_net}
Let $L \in \mathbb{N}$ and let $\Phi^{1}=\left(\left(A_{1}^{1}, b_{1}^{1}\right), \ldots,\left(A_{L}^{1}, b_{L}^{1}\right)\right)$ and $\Phi^{2}=\left(\left(A_{1}^{2}, b_{1}^{2}\right), \ldots,\left(A_{L}^{2}, b_{L}^{2}\right)\right)$ be
two neural networks with $L$ layers and with $d$-dimensional input. We define
$$
{P}\left({\Phi}^{1}, \Phi^{2}\right):=\left(\left(\widetilde{A}_{1}, \widetilde{b}_{1}\right), \ldots,\left(\tilde{A}_{L}, \widetilde{b}_{L}\right)\right)
$$
where
$$
\tilde{A}_{1}:=\left(\begin{array}{c}
A_{1}^{1} \\
A_{1}^{2}
\end{array}\right), \quad \tilde{b}_{1}:=\left(\begin{array}{c}
b_{1}^{1} \\
b_{1}^{2}
\end{array}\right) \quad \text { and } \quad \tilde{A}_{\ell}:=\left(\begin{array}{cc}
A_{\ell}^{1} & 0 \\
0 & A_{\ell}^{2}
\end{array}\right), \quad \tilde{b}_{\ell}:=\left(\begin{array}{c}
b_{\ell}^{1} \\
b_{\ell}^{2}
\end{array}\right) \quad \text { for } 1<\ell \leq L.
$$
Then, ${P}\left(\Phi^{1}, \Phi^{2}\right)$ is a neural network with $d$-dimensional input and $L$ layers, called the parallelization of $\Phi^{1}$ and $\Phi^{2}$.

\end{definition}

\par

\begin{lemma}
Let $L, d\in \mathbb{N}$, $\Phi^{1}$ and  $\Phi^{2}$ be two neural networks with $L$ layers and with $d$-dimensional input.  Then,
$M\left(P\left(\Phi^{1}, \Phi^{2}\right)\right)=M\left(\Phi^{1}\right)+M\left(\Phi^{2}\right)$ ,
and

$$
{R}_{\rho_3}\left(\mathrm{P}\left(\Phi^{1}, \Phi^{2}\right)\right)(x)=\left({R}_{\rho_3}\left(\Phi^{1}\right)(x), {R}_{\rho_3}\left(\Phi^{2}\right)(x)\right), \quad \text {for any } x \in \mathbb{R}^{d}.
$$
\end{lemma}

\par

\begin{proof}
The proof is straightforward and therefore is left for the reader.
\end{proof}

\section{Approximations with deep ReCU neural networks in mixed Sobolev space}\label{sec:main}

We are interested
in approximating functions in subsets of the Sobolev space
$W^{n, p}_{k, q}((0,1), (0,1)^d)$ with realizations of neural networks. For
this we define the set:
\begin{equation}\label{eq:U}
\mathcal{U}_{k,q,n,p,d, B} :=
	\left\{
	u\in	W^{n, p}_{k, q}((0,1), (0,1)^d):
	\Vert u\Vert_{W^{n, p}_{k, q}((0,1), (0,1)^d)}\leq B
	\right\}.
\end{equation}

Next, we construct a partition of unity that can be defined as a product
of piecewise linear functions, such that each factor of the product can
be realized by a neural network.

\begin{lemma}\label{lemma:partition_of_unity}
For any $d,N\in \mathbb{N}$ there exists a collection of functions 
\[
\Psi=\left\{\phi_\mu:\mu\in\{0,\dots,N\}^{d+1}\right\}
\]
with $\phi_\mu:\mathbb{R}\times\mathbb{R}^d\to\mathbb{R}$ for all $\mu\in\{0,\dots,N\}^{d+1}$
	with the following properties:
\begin{enumerate}[(i)]
\item \label{item:pou_01} $0\leq \phi_\mu(t, x)\leq 1$ for every $\phi_\mu \in\Psi$ and
		every $(t,x)\in \mathbb{R}\times\mathbb{R}^d$; 
\item  \label{item:pou_sum} $\sum_{\phi_\mu\in\Psi}\phi_\mu(t, x)=1$ for every
	$(t, x) \in[0,1]\times[0,1]^d$;
\item \label{item:pou_supp} $supp\, \phi_\mu\subset B_{\frac{1}{N},
	\Vert{\cdot}\Vert_{\ell^\infty}}(\frac{\mu}{N})$ for every $\phi_\mu \in \Psi$;
\item \label{item:pou_derivative}  there exists a constant $c\geq 1$ such that
	$\Vert{\phi_\mu}\Vert_{W_{k, \infty}^{n, \infty}(\mathbb{R},\mathbb{R}^d)}
		\leq (c\cdot N)^{n+k}$ for $k, n\in\{0,1\}$;
\item \label{item:pou_network} there exist absolute constants $C,c\geq 1$
	such that for each $\phi_\mu\in\Psi$ there is a neural network $\Phi_\mu$ with
	$d+1$-dimensional input and $d+1$-dimensional output, with at most three layers,
	$C (d+1)$ nonzero weights and neurons, that satisfies
\[
\prod_{l=0}^d R_{\rho_3}(\Phi_\mu)(x_l) =
\prod_{l=0}^d [R_{\rho_3}(\Phi_\mu)]_l(t,x)
=\phi_\mu(t, x), \quad  \text{ where } x_0  = t
\]
and $\Vert[R_{\rho_3}(\Phi_\mu)]_l\Vert_
{W_{k, \infty}^{n, \infty}((0,1),(0,1)^d)}\leq (cN)^{n +k}$
for all $l=0,\ldots,d$ such that $k, n\in\{0,1\}$.
\end{enumerate}
\end{lemma}

\par

\begin{proof}
As in \cite{yarotsky2017error}, we define the functions
\[
\psi:\mathbb{R}\to\mathbb{R},\qquad\psi(x):=
\begin{cases}
54 + 81 x + \frac{81}2 x^2 + \frac {27}4 x^3 , &x \in [-2, -\frac {5}{3}],
\\
-\frac {17}2 - \frac {63}2 x - 27 x^2 - \frac{27}{4}x^3, & x \in [-\frac {5}{3}, -1],
\\
\frac 14 (20 + 36 x + 54 x^2 + 27 x^3), &x\in [-1, -\frac 23],
\\
1, & x \in (-\frac{2}3, \frac 23),
\\
\frac 14(20 - 36 x + 54 x^2 - 27 x^3),& x \in  [\frac 23, 1],
\\
-\frac {17}2 + \frac {63}2 x - 27 x^2 + \frac{27}4 x^3, & x \in [1, \frac 53],
\\
54 - 81 x + \frac {80}2 x^2 - \frac{27}4 x^3, &x \in [\frac 53, 2]
\\
0, &x \in \mathbb{R}\setminus (-2, 2).
\end{cases}
\]

\begin{figure}[!htb]
	    \begin{tikzpicture}
        \begin{axis}
    
        \addplot [domain=-4:4, samples=100] {2/3*(max(0, 3/2*x +3)^3 -2*max(0, 3/2*x +5/2)^3 +2*max(0, 3/2*x +3/2)^3 -max(0, 3/2*x +1)^3-max(0, 3/2*x -1)^3 +2*max(0, 3/2*x -3/2)^3 - 2*max(0, 3/2*x -5/2)^3 +max(0, 3/2*x -3)^3)};
        \addlegendentry{$\psi$}
    
        \end{axis}
    \end{tikzpicture}
	\caption{\label{fig:bump} The bump function $\psi(x)$}
\end{figure}
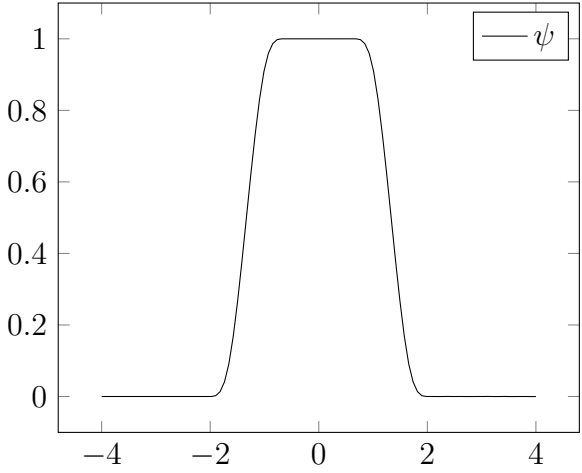

The function $\phi_\mu:\mathbb{R}\times\mathbb{R}^d\to\mathbb{R}$
is a product of scaled and shifted versions of $\psi$.
Concretely, we set
\begin{equation}\label{eq:partition_unity}
\phi_\mu(t, x):=\psi\left(3N\left(t-\frac{\mu_0}{N}\right)\right)
\times\prod_{l=1}^d \psi\left(3N\left(x_l-\frac{\mu_l}{N}\right)\right),
\end{equation}
for $\mu=(\mu_0,\mu_1,\ldots,\mu_d)\in\{0,\dots,N\}^{d+1}$.
Then, (\ref{item:pou_01}),(\ref{item:pou_sum}) and (\ref{item:pou_supp})
follow easily from the definition.

To show (\ref{item:pou_derivative}), note that
$\Vert{\phi_\mu}\Vert_{L_t^\infty L_x^\infty}\leq 1$
follows already from (\ref{item:pou_01}). 
We need to show that the claim holds for $\frac\partial{\partial x_l}\phi_\mu(t,x)$,
$\frac\partial{\partial t}\phi_\mu(t, x)$ and 
$\frac\partial{\partial x_l}\frac\partial{\partial t}\phi_\mu(t,x)$.
For this, let $l\in\{1,\ldots,d\}$ $t\in \mathbb{R}$, and $x\in\mathbb{R}^d$,
then, using the fact that $|\psi(x)'|\leq \frac 92$ for any $x \in \mathbb{R}$, we get
\begin{align*}
\left\vert{\frac\partial{\partial x_l}\phi_\mu(x)}\right\vert&=
\left\vert{\psi\left(3N\left(t-\frac{\mu_0}{N}\right)\right)\prod_{i=1,i\neq l}^d\psi
\left(3N\left(x_i-\frac{\mu_l}{N}\right)\right)}\right\vert
\left\vert{\psi'\left(3N\left(x_l-\frac{\mu_l}{N}\right)\right)3N}\right\vert
\\
		&\leq c N,
\\
\left\vert{\frac\partial{\partial t}\phi_\mu(x)}\right\vert&=
\left\vert{\psi\left(3N\left(t-\frac{\mu_0}{N}\right)\right)\prod_{i=1}^d\psi
\left(3N\left(x_i-\frac{\mu_l}{N}\right)\right)}\right\vert
\left\vert{\psi'\left(3N\left(t-\frac{\mu _0}{N}\right)\right)3N}\right\vert
\\
		&\leq cN,
\\
\left\vert{\frac\partial{\partial x_l}\frac\partial{\partial t}\phi_\mu(x)}\right\vert&=
\left\vert{\psi\left(3N\left(t-\frac{\mu_0}{N}\right)\right)\prod_{i=1,i\neq l}^d\psi
\left(3N\left(x_i-\frac{\mu_l}{N}\right)\right)}\right\vert
\\
&\qquad\qquad\times\left\vert{\psi'\left(3N\left(t-\frac{\mu _0}{N}\right)\right)3N}\right\vert
\left\vert{\psi'\left(3N\left(x_l-\frac{\mu_l}{N}\right)\right)3N}\right\vert
\\
		&\leq (cN)^2,
\end{align*}
where $c> \frac {27}2$ is a suitable constant.
It follows that $\Vert{\phi_\mu}\Vert_{W_{k, \infty}^{n,\infty}}\leq (cN)^{n+k}$.

The proof of (\ref{item:pou_network}), is given by constructing a network $\Phi_\psi$
that realizes the function $\psi$. Thus, let 
\[
	A_1:=\frac 32\left[ \begin{array}{c}
	1\\[1ex]
	1\\[1ex]
	1\\[1ex]
	1\\[1ex]
	1\\[1ex]
	1\\[1ex]
	1\\[1ex]
	1 \end{array}\right],
	b_1:=\frac 12\left[ \begin{array}{c}
	6\\[1ex]
	5\\[1ex]
	3\\[1ex]
	2\\[1ex]
	-2\\[1ex]
	-3\\[1ex]
	-5\\[1ex]
	-6
\end{array} \right] \;\text{and } 
	A_2:=\frac 23 \left[ \begin{array}{c|c c c c c c c c}
	0&1&-2&2&-1&-1&2&-2&1
	\end{array} \right],\;
	b_2:=0,
\]
and $\Phi_{\psi}:=((A_1,b_1),(A_2,b_2))$. Then $\Phi_{\psi}$ is a two-layer
network with one-dimensional input and one-dimensional output, with $24$
nonzero weights and $10$ neurons such that
\[
R_{\rho_3}(\Phi_{\psi})(x)=\psi(x)\quad\text{for all}\quad x\in\mathbb{R}.
\]
The remaining part of the proof is similar to \cite[Lemma C.3$(\mathrm{v})$]{GuhKutPet}.
The details are left to the reader.
\end{proof}

\par

Next we show that any function in the space $W_{m,p}^{m,p}((0,1), (0,1)^d)$, can be approximated
by a sum of localized polynomial of degree at most $m-1$. This makes  Lemma \ref{lemma:polynomial_approximation}
one of the main ingredient in our strategy to proof the main result in Theorem \ref{thm:main}.
\par

\begin{lemma}\label{lemma:polynomial_approximation}
Let $d,N\in \mathbb{N}$, $n, k\in \{0,1\}$, and $m\in\mathbb{N}$
such that $m \geq n+k+ 1$,
$1\leq p, q\leq \infty$ and
$\Psi=\Psi(d+1,N)=\left\{\phi_\mu:\mu\in \{0,\dots,N\}^{d+1}\right\}$
be the partition of unity
from Lemma \ref{lemma:partition_of_unity}.
Then there is a constant $C=C(d+1,m,p)>0$ such that for any
$u \in W_{m, p}^{m,p}((0,1),(0,1)^d)$, there exist polynomials
$p_{u,\mu}(t, x)=\sum_{|\alpha|+\kappa\leq m-1}c_{\mu,\kappa, \alpha}
t^\kappa x^\alpha$ for $\mu\in \{0,\dots,d\}^{d+1}$
with the following properties:
	
Let $u_N:=\sum_{\mu\in \{0,\dots,N\}^{d+1}}\phi_\mu p_{u,\mu}$,
then the operator
$$T_N:W_{m,p}^{m,p}((0,1),(0,1)^d)\to W_{k,q}^{n,p}((0,1),(0,1)^d)$$
with $T_ N  u=u-u_N$ is linear and bounded with
\[
\Vert{T_ N  u}\Vert_{W_{k,q}^{n,p}((0,1),(0,1)^d)}
\leq C\left(\frac{1}{N}\right)^{m- n-k}\Vert{u}\Vert_{W_{m,p}^{m,p}((0,1),(0,1)^d)}.
\]
Furthermore, there is a constant $c=c(m, d+1)>0$ such that for any
$u\in W_{m,p}^{m,p}((0,1),(0,1)^d)$ the coefficients of the polynomials $p_{u,\mu}$ satisfy 
\[
\left|{c_{\mu,\kappa, \alpha}}\right|
\leq cN^{(d+1)/p}\Vert{U}\Vert_{W_{m, p}^{m, p}(\Omega_{\mu,N})}
\]
for all $\kappa \in \mathbb{N}_0$, $\alpha\in\mathbb{N}^d_0$ with
$|{\alpha}| + \kappa \leq m-1$ and $\mu\in\{0,\ldots,N\}^{d+1}$, where
$\Omega_{\mu,N}:=B_{\frac{1}{N},\Vert{\cdot}\Vert_{\ell^\infty}}\left(\frac{m}{N}\right)$
and $U\in W_{m,q}^{m,p}(\mathbb{R}, \mathbb{R}^d)$ is an extension of $u$.
\end{lemma}

\par

\begin{proof}
We proceed in a similar way as the proof of
\cite[Theorem 1]{GuhKutPet}.
Thus, we need the notion of the extension operator on the anisotropic Sobolev
space cf. \cite{Shvartsman}, which is a  generalization of Stein theorem for the extension
operator to anisotropic Sobolev spaces.
That is, we can define the operator
$E: W_{k, q}^{n,p }((0,1), (0,1)^d)\to W_{k, q}^{n,p }(\mathbb{R}, \mathbb{R}^d)$
as the extension operator.

Moreover, we use approximation properties of averaged Taylor polynomials
(see Bramble-Hilbert Lemma~\ref{lemma:bramble_hilbert}) to derive local estimates
and then combine them using a partition of unity to obtain a global estimate.
Following similar approach as in \cite{GuhKutPet}

\par

\textbf{Step 1 (Averaged Taylor polynomials)}:
Let $ U:=Eu$ and  $\mu\in \{0,\dots,N\}^{d+1}$, we set
\[
\Omega_{\mu,N}:=B_{\frac{1}{N},\Vert{\cdot}\Vert_{\ell^\infty}}
\Big(\frac{\mu}{N}\Big) \quad \text{and}\quad 
B_{\mu,N} := B_{\frac{3}{4N},|{\cdot}|}\Big(\frac{\mu}{N}\Big),
\]
and $p_\mu=p_{u,\mu}$ stands for the Taylor polynomial of order $m$ of $U$
averaged over $B_{\mu,N}$ (cf.\ Definition~\ref{def:taylor}).
It follows from Lemma~\ref{prop:taylor_is_polynom}
that we can write
$p_\mu=\sum_{|\alpha|+ \kappa \leq m-1}c_{\mu,\kappa,\alpha}t^\kappa x^\alpha$,
such that for $c'=c'(m,d+1)>0$, we have
\[
\left|{c_{\mu,\kappa, \alpha}}\right|
\leq c'\left(\frac{3}{4N}\right)^{-(d+1)/p}
\Vert{U}\Vert_{W_{m, p}^{m, p}(\Omega_{\mu,N})}
\leq c''N^{(d+1)/p}\Vert{U}\Vert_{W_{m, p}^{m, p}(\Omega_{\mu,N})},
\]
where $c''$ is a nonnegative constant depends on $m$ and $d+1$.

\par

\textbf{Step 2 (Local estimates in $\Vert{\cdot}\Vert_{W_{k,q}^{n,p}}, k,n \in\{0,1\}$)}:
To check that the conditions of the Bramble-Hilbert Lemma \ref{lemma:bramble_hilbert}
are fulfilled, note that $B_{\mu,N}\subset\subset\Omega_{\mu,N}$.
Furthermore, $B_{\mu,N}$ is a ball in $\Omega_{\mu,N}$ such that $\Omega_{\mu,N}$
is star-shaped with respect to $B_{\mu,N}$. Moreover,
$diam_{|{\cdot}|}(\Omega_{\mu,N})
=(S\sqrt {d+1})$ where $S= 2/N$, $r^\star_{\text{max}}(\Omega_{\mu,N})=1/N$
and,
$r_{|{\cdot}|}\left(B_{\mu,N}\right)
>
\frac{1}{2}\cdot r^\star_{\text{max}}(\Omega_{\mu,N})$.
Finally, the chunkiness parameter of $\Omega_{\mu,N}$
\begin{equation}\label{eq:chunkiness}
\gamma(\Omega_{\mu,N})=
diam(\Omega_{\mu,N})\cdot\frac{1}{r^\star_{\text{max}}(\Omega_{\mu,N})}
=\frac{2\sqrt {d+1}}{N}\cdot N=2\sqrt {d+1}.
\end{equation}
Applying the Bramble-Hilbert Lemma \ref{lemma:bramble_hilbert} yields
for each $\mu\in\{0,\ldots,N\}^{d+1}$ the local estimate
\begin{align*}
\Vert{U-p_\mu}\Vert_{L_t^qL_x^p (\Omega_{ \mu,N})}&\leq
C\left(\frac{2\sqrt {d+1}}{N}\right)^m 
\left\Vert{U}\right\Vert_{W_{m, p}^{m, p}(\Omega_{m,N})}
\leq
\tilde C\left(\frac{1}{N}\right)^m
\Vert{U}\Vert_{W_{m,p}^{m, p}(\Omega_{m,N})},\numberthis\label{eq:local_Linf_bound}
\end{align*}
where $C$ depends on $m$ and $d$ (see Lemma \ref{lemma:bramble_hilbert}),
since the chunkiness parameter of $\Omega_{\mu,N}$
is a constant depending only on $d$ (see \eqref{eq:chunkiness})
and~$\tilde C = \tilde C(m,d)>0$.
Similarly, we get
\begin{equation}\label{eq:local_Winf_bound}
\begin{aligned}
\left|{U-p_\mu}\right|_{W_{0, q}^{1, p}(\Omega_{ \mu,N})}
&\leq
c_1\left(\frac{1}{N}\right)^{m-1}\Vert{U}\Vert_{W_{m,p}^{m, p}(\Omega_{\mu,N})},
\\
\left|{U-p_\mu}\right|_{W_{1, q}^{0, p}(\Omega_{ \mu,N})}
&\leq
c_2\left(\frac{1}{N}\right)^{m-1}\Vert{U}\Vert_{W_{m,p}^{m, p}(\Omega_{\mu,N})},
\\
\left|{U-p_\mu}\right|_{W_{1, q}^{1, p}(\Omega_{ \mu,N})}
&\leq
c_3\left(\frac{1}{N}\right)^{m-2}\Vert{U}\Vert_{W_{m,p}^{m, p}(\Omega_{\mu,N})},
\end{aligned}
\end{equation}
where $c_i$ is a suitable constant depends on $m$ and $d$ for $i\in \{1,2,3\}$.

Combining Lemma \ref{lemma:partition_of_unity}, 
inequalities \eqref{eq:local_Linf_bound} and \eqref{eq:local_Winf_bound} using
the cut-off functions from the partition of unity, we get
\begin{align*}
\Vert{ \phi_\mu(U-p_\mu)}\Vert_{L _t^q L _x^p(\Omega_{\mu, N})}&
\leq
\Vert{\phi_\mu}\Vert_{L _t^\infty L _x^\infty(\Omega_{\mu, N})}\cdot
\Vert{ U-p_\mu}\Vert_{L _t^q L _x^p(\Omega_{\mu, N})}
\\
&\leq \tilde{C}\left(\frac{1}{N}\right)^m
\Vert{U}\Vert_{W_{m,p}^{m, p}(\Omega_{\mu, N})}
\numberthis\label{eq:part_unity_Linf_bound}.
\end{align*}
Next we use the product inequality for weak derivatives
from Lemma \ref{lemma:product_rule_bound_p}.
Under this consideration, there are constants
$C_1, C_2$ depend on $d$ and $p$ such that
\begin{equation}\label{eq:0t1x_prod_estim}
\begin{aligned}
|{\phi_\mu(U-p_\mu)}|_{W_{0,q}^{1,  p}(\Omega_{\mu,N})}&\leq
C_1 \left(|\phi_\mu|_{W_{0, \infty}^{1,\infty}(\Omega_{\mu,N})}
\Vert{U-p_\mu}\Vert_{L_t^q L_x^p(\Omega_{\mu,N})}
+ \Vert{\phi_\mu}\Vert_{L _t^\infty L_x^\infty(\Omega_{\mu,N})}
	\vert{U-p_\mu}\vert_{W_{0,q}^{1, p}(\Omega_{\mu,N})}\right)
\\[1ex]
&\leq
C_1 \cdot c N\cdot \tilde{C}\left(\frac{1}{N}\right)^{m}
\Vert U\Vert_{W_{m, p}^{m, p}(\Omega_{\mu,N})}
+ C_1\cdot c_1  \left(\frac{1}{N}\right)^{m-1}
\Vert U\Vert_{W_{m, p}^{m, p}(\Omega_{\mu,N})}
\\[1ex]
&\leq c_4 \left(\frac{1}{N}\right)^{m-1}\Vert U\Vert_{W_{m, p}^{m, p}(\Omega_{\mu,N})}
\end{aligned}
\end{equation}
where $c_4 = c_4(m,d,p)>0$, such that the first part of the second inequality follows
from Lemma \ref{lemma:partition_of_unity}(\ref{item:pou_derivative}),
with \eqref{eq:local_Linf_bound} and the second part from
Lemma~\ref{lemma:partition_of_unity}~(\ref{item:pou_derivative}),
together with \eqref{eq:local_Winf_bound}. 

In a similar way we get the following results
\begin{equation}\label{eq:1t0x_prod_estim}
\begin{aligned}
|{\phi_\mu(U-p_\mu)}|_{W_{1,q}^{0,  p}(\Omega_{\mu,N})} &\leq
\left|{\phi_\mu}\right|_{W_{1,\infty}^{0, \infty}(\Omega_{\mu,N})}
 \left\Vert{U-p_\mu}\right\Vert_{L_t^qL_x^p(\Omega_{\mu,N})}
+\left\Vert{\phi_\mu}\right\Vert_{L_t^\infty L_x^\infty(\Omega_{\mu,N})}
		\left|{U-p_\mu}\right|_{W_{1, q}^{0,p}(\Omega_{\mu,N})}
\\[1ex]
&\leq cN\cdot  \tilde{C}\left(\frac{1}{N}\right)^{m}
\Vert U\Vert_{W_{m, p}^{m, p}(\Omega_{\mu,N})}
+c_2  \left(\frac{1}{N}\right)^{m-1} \Vert U\Vert_{W_{m, p}^{m, p}(\Omega_{\mu,N})}
\\[1ex]
&\leq c_5  \left(\frac{1}{N}\right)^{m-1} \Vert U\Vert_{W_{m, p}^{m, p}(\Omega_{\mu,N})}
\end{aligned}
\end{equation}
\begin{equation}\label{eq:1t1x_prod_estim}
\begin{aligned}
|{\phi_\mu(U-p_\mu)}|_{W_{1,q}^{1,  p}(\Omega_{\mu,N})}&\leq
C_2\left( |{\phi_\mu}|_{W_{1, \infty}^{1, \infty}(\Omega_{\mu,N})}
\Vert{U-p_\mu}\Vert_{L_t^qL_x^p(\Omega_{\mu,N})}
+|{\phi_\mu}|_{W_{1, \infty}^{0, \infty}(\Omega_{\mu,N})}
\vert{U-p_\mu}\vert_{W_{0, q}^{1,p}(\Omega_{\mu,N})}\right.
\\
&\qquad\left.
+|{\phi_\mu}|_{W_{0, \infty}^{1, \infty}(\Omega_{\mu,N})}
|{U-p_\mu}|_{W_{1, q}^{0, p}(\Omega_{\mu,N})}
+ \Vert{\phi_\mu}\Vert_{L_t^\infty L_x^\infty(\Omega_{\mu,N})}
\vert{ U-p_\mu}\vert_{W_{1,q}^{1, p}(\Omega_{\mu,N})}\right)
\\[1ex]
& \leq C_2\left(
 (cN)^2 \cdot \tilde{C}\left(\frac{1}{N}\right)^{m}
 \Vert U\Vert_{W_{m, p}^{m, p}(\Omega_{\mu,N})}
+ cN \cdot c_1\left(\frac{1}{N}\right)^{m-1}
\Vert{U}\Vert_{W_{m,p}^{m, p}(\Omega_{\mu,N})}
\right.
\\
&\left.+ cN \cdot  c_2\left(\frac{1}{N}\right)^{m-1}
\Vert{U}\Vert_{W_{m,p}^{m, p}(\Omega_{\mu,N})}
+ c_2 \left(\frac{1}{N}\right)^{m-2}
\Vert{U}\Vert_{W_{m,p}^{m, p}(\Omega_{\mu,N})}
\right)
\\[1ex]
&\leq c_6 \left(\frac{1}{N}\right)^{m-2}
\Vert{U}\Vert_{W_{m,p}^{m, p}(\Omega_{\mu,N})}.
\end{aligned}
\end{equation}

Now it easily follows from \eqref{eq:part_unity_Linf_bound}, \eqref{eq:0t1x_prod_estim}
\eqref{eq:1t0x_prod_estim} and \eqref{eq:1t1x_prod_estim} that
\begin{equation}\label{eq:part_unity_Winf_bound}
\Vert{\phi_\mu(U-p_\mu)}\Vert_{W_{1, q}^{1, p}(\Omega_{\mu,N})}
\leq C \left(\frac{1}{N}\right)^{m-2}\Vert{U}\Vert_{W_{m,p}^{m, p}(\Omega_{\mu,N})},
\end{equation}
for some constant $C=C(m,d,p)>0$.

\textbf{Step 3 (Global estimate in $\Vert{\cdot}\Vert_{W_{k,q}^{n,p}}, k, n\in\{0,1\}$)}:
To derive the global estimate, we start by noting that with property (\ref{item:pou_sum})
from Lemma \ref{lemma:partition_of_unity} we have
\begin{equation}\label{eq:sumup}
U(t, x)=\sum_{ \mu \in\{0, \dots,N \}^{d+1} }\phi_\mu(t, x) U(t, x),
\quad \text{for a.e.\ }(t, x)\in(0,1)\times(0,1)^d.
\end{equation}
Let $k,n \in\{0,1\}$, we have
\begin{align*}
\Vert{u -\sum_{\mu\in\{0, \dots,N \}^{d+1} }
\phi_\mu p_\mu}\Vert_{ W_{k,q}^{n, p}((0,1),(0,1)^d)}&=
\Vert{\sum_{\mu\in\{0, \dots,N \}^{d+1} }
\phi_\mu(U-p_\mu)}\Vert_{ W_{k,q}^{n, p}((0,1),(0,1)^d)}
\\
&\leq\sum_{\widetilde \mu\in\{0, \dots,N \}^{d+1} }
\Vert{\sum_{\mu\in\{0, \dots,N \}^{d+1} }
\phi_\mu(U-p_\mu)}\Vert_{ W_{k,q}^{n, p}(\Omega_{\widetilde \mu,N})},
\numberthis\label{eq:global_max}
\end{align*}
where in the first step we use the fact that $U$ is an extension of
$u$ on $(0,1)\times(0,1)^d$,
the last step follows from the fact that
$(0,1)\times(0,1)^d\subset\bigcup_{\widetilde \mu\in\{0, \dots,N \}^{d+1} }
\Omega_{\widetilde \mu,N}$.
Consequently, for each $\widetilde \mu\in\{0, \dots,N \}^{d+1} $, we get
\begin{align*}
\Vert{\sum_{\mu\in\{0, \dots,N \}^{d+1} }
\phi_\mu(U-p_\mu)}\Vert_{ W_{k,q}^{n, p}(\Omega_{\widetilde m\mu,N})}&\leq
\sum_{\substack{\mu\in\{0, \dots,N \}^{d+1} ,\vspace{0.2em}\vspace{0.2em}
\\
\Vert{\mu-\widetilde \mu}\Vert_{\ell^\infty}
\leq 1}} \Vert{\phi_\mu(U-p_\mu)}\Vert_{ W_{k, q}^{n, p}(\Omega_{\widetilde \mu,N})}
\\
&\leq \sum_{\substack{\mu\in\{0, \dots,N \}^{d+1} ,\vspace{0.2em}\vspace{0.2em}
\\
\Vert{\mu-\widetilde \mu}\Vert_{\ell^\infty} \leq 1}}
\Vert{\phi_\mu(U-p_\mu)}\Vert_{ W_{k,q}^{n, p}(\Omega_{\mu,N})}
\\
&\leq C \left(\frac{1}{N}\right)^{m-n-k}
\sum_{\substack{\mu\in\{0, \dots,N \}^{d+1} ,\vspace{0.2em}
\\
\Vert{m-\widetilde m}\Vert_{\ell^\infty} \leq 1}}
\Vert{U}\Vert_{ W_{m, p}^{m, p}(\Omega_{\mu,N})}
\numberthis\label{eq:max_neighbors}
\end{align*}
where first and second steps follow from the support property~(\ref{item:pou_supp})
from~Lemma~\ref{lemma:partition_of_unity},
third step follows from \eqref{eq:part_unity_Linf_bound}, \eqref{eq:0t1x_prod_estim},
\eqref{eq:1t0x_prod_estim} and from \eqref{eq:1t1x_prod_estim} 
for $(k=n=0)$, $(k=0, n=1)$, $(k=1, n=0)$, and for $(k=n=1)$ respectively.
Here $C>0$ depends on $m$, $d$ and $p$.

Using the fact that  $u_N:=\sum_{\mu\in \{0,\dots,N\}^{d+1}}\phi_\mu p_{\mu}$,
\eqref{eq:global_max} and with Equation \eqref{eq:max_neighbors},
we get the following bound
\begin{align*}
\Vert{u-u_N}\Vert_{ W_{k, q}^{n, p}((0,1),(0,1)^d)}&\leq
\sum_{\widetilde\mu\in\{0, \dots,N \}^d }
C \left(\frac{1}{N}\right)^{m- n-k}
\sum_{\substack{\mu\in\{0, \dots,N \}^{d+1} ,\vspace{0.2em}
\\
\Vert{\mu-\widetilde \mu}\Vert_{\ell^\infty} \leq 1}}
\Vert{u}\Vert_{ W_{m, p}^{m, p}(\Omega_{\mu,N})}
\\
&\leq
C \left(\frac{1}{N}\right)^{m-n-k} \sum_{\widetilde \mu\in\{0, \dots,N \}^{d+1} }
\sum_{\substack{\mu\in\{0, \dots,N \}^{d+1} ,\vspace{0.2em}
\\
\Vert{\mu-\widetilde \mu}\Vert_{\ell^\infty} \leq 1}}
\Vert{U}\Vert_{ W_{m,p}^{m, p}(\Omega_{\mu,N})}
\\
&\leq C \left(\frac{1}{N}\right)^{m-n-k} 3^d
\sum_{\widetilde \mu\in\{0, \dots,N \}^d }
\Vert{U}\Vert_{ W_{m, p}^{m, p}(\Omega_{\widetilde \mu,N})}
\\
&\leq C\left(\frac{1}{N}\right)^{m-n-k} 3^d 2^d
\Vert{U}\Vert_{ W_{m ,p}^{m, p}(\bigcup_{\widetilde \mu\in\{0, \dots,N \}^{d+1} }
\Omega_{\widetilde \mu,N})}^p,
\end{align*}
where the last two steps follow from the definition of $\Omega_{\widetilde \mu,N}$.
Thus, we have
\[
\left\Vert{u-u_N}\right\Vert_{ W_ {k,q}^{n, p}((0,1),(0,1)^d)}
\leq C_7\left(\frac{1}{N}\right)^{m-n-k}
	\Vert{U}\Vert_{ W_{m, p}^{m, p}(\mathbb{R}, \mathbb{R}^d)}
\leq \tilde{C_7} \left(\frac{1}{N}\right)^{m-n-k}
	\Vert{u}\Vert_{ W_{m, p}^{m, p}((0,1), (0,1)^d)}
\]
for $k, n\in\{0,1\}$, where the extension operator continuity was used in
the first and second step. Here $C_7$ and
$ \tilde{C_7}$ are positive constants depend on $m,d$ and $p$.
%
%
\end{proof}

\par

\begin{remark}\label{rem:recu_represent_requ}
The function $f(x)=x^2$ can be represented by ReCU neural network in a compact interval.
Indeed let $r>0$, 
\[
A_{1}:=\left[\begin{array}{l}
-1 \\
1 
\end{array}\right], \quad b_{1}:=r\left[\begin{array}{r}
1 \\
1
\end{array}\right] \quad \text { and } \quad
A_{2}:=\frac 1{6r} \left[\begin{array}{c| c c}
0 & 1 & 1
\end{array}\right], \quad
b_{2}:=-\frac {r^2}3
\]
and $\Phi_{x^2,r}:=\left(\left(A_{1}, b_{1}\right),\left(A_{2}, b_{2}\right)\right)$.
Then $\Phi_{x^2,r}$ is a two-layer network with one-dimensional input and one-dimensional output,
with 7 nonzero weights and 4 neurons such that
\[
R_{\rho_3}\left(\Phi_{x^2,r}\right)(x)=x^2 \quad \text { for any } x \in [-r , r].
\]
\end{remark}

\par

\begin{remark}\label{rem:recu_approximate_product}
	The product $tx$ can be represented by two-layer ReCU network with two-dimensional input
	and one-dimensional output, 16 nonzero weights and 7 neurons. Indeed, let
	$\times_r= \left( (A_1, b_1), (A_2, b_2)\right)$ where
	\[A_1 = \begin{pmatrix}
		-1 & -1\\
		1 & 1\\
		-1 & 1\\
		1 & -1		
	   \end{pmatrix}, 
	   b_1 =2r \begin{pmatrix}
		1\\
		1\\
		1\\
		1
	\end{pmatrix},
	A_2 = \frac 1{48r}\begin{pmatrix}
		1\\
		1\\
		-1\\
		-1
	\end{pmatrix} \text{ and } b_2=0.
	\]
	Hence, $R_{\rho_3}(\times_r)(t,x) = tx$ such that $t, x  \in [-r,r]$ and $r>0$. Moreover,
	if  $n,k \in \{0,1\}$, then 
	\begin{equation}\label{eq:rho3_bound_derivative}
		\vert{R_{\rho_3}(\times_r)}\vert_{{W_{k, \infty}^{n, \infty}}
			((-r,r),(-r,r))}= r^{2-k-n}.
	\end{equation}

\end{remark}
\par

Using (\ref{item:pou_network}) from Lemma~\ref{lemma:partition_of_unity},
a localized (mixed) monomial $\phi_\mu (t, x) t^\kappa x^\alpha$
can be expressed by the product of the output components of a network
$\Phi_{(\mu,\alpha, \kappa)}$ as follows:
\begin{equation}\label{eq:mult_variate_loc_mon_identity}
\phi_\mu(t, x) \, t^\kappa x^\alpha=\prod_{l=1}^m 
	\big[R_{\rho_3}(\Phi_{(\mu,\alpha, \kappa)})\big]_l(t, x).
\end{equation}

In the following lemma we show that the localized monomials
\eqref{eq:mult_variate_loc_mon_identity}
can be approximated by ReCU neural networks, using the fact that
$R_{\rho_3}\left(\Phi_{x^2,1}\right)= x^2$
on $(0, 1)$ cf. Remark \ref{rem:recu_represent_requ}.

\par

\begin{lemma}\label{lemma:network_multiplikation}
Let $d,\mu,K\in\mathbb{N}$ and $N\geq 1$ be arbitrary.
Then there is a constant $C=C(\mu)>0$ such that the following holds:

For any $\epsilon\in (0, 1/2)$, and any neural network $\Phi$ with $(d+1)$-dimensional
input and $m$-dimensional output where $m\leq \mu$, and with number of layers,
 neurons and weights all bounded by $K$, such that 
\[
\Vert{[R_{{\rho_3}}(\Phi)]_l}\Vert_{W_{k,\infty}^{n, \infty}((0,1),(0,1)^d)}\leq N^{k+n} \quad
\text{for}
\quad n, k\in\{0,1\}\text{ and }l=1,\ldots,m
	\]
there exists a neural network $\Psi_{\epsilon,\Phi}$ with $(d+1)$-dimensional input
and one-dimensional output, and with number of layers, neurons and weights
all bounded by $CK$, such that
\begin{equation}\label{eq:approximation_s}
\Vert{R_{{\rho_3}}(\Psi_{\epsilon,\Phi})-\prod_{l=1}^m
[R_{{\rho_3}}(\Phi)]_l}\Vert_{W_{k, \infty}^{n,\infty}((0,1),(0,1)^d)} \leq c  N^{k+n} \epsilon
\end{equation}
for $n, k\in\{0,1\}$ and some constant $c=c(d+1,\mu,k,n)$.
Moreover, for $t\in (0,1),x\in(0,1)^d$, we have
\begin{equation}\label{eq:zero_multiplication}
R_{{\rho_3}}(\Psi_{\epsilon,\Phi})(t, x)=0\quad \text{if}\quad
	\prod_{l=1}^m [R_{{\rho_3}}(\Phi)]_l(t, x)=0.
\end{equation}

\end{lemma}

\par

\begin{proof}
We show the proof by induction over $\mu\in\mathbb{N}$. Moreover, we will
make sure that the constant $c$ in \eqref{eq:approximation_s} can be written as
$c=\mu^{2-k-n}c_1^{k+n}$, where $c_1$ depends on the dimension $d+1$ and $\mu$.
Furthermore, we show that 
the first  $L(\Phi)-1$ layers of $\Psi_{\epsilon,\Phi}$ and $\Phi$ coincide, and that
\begin{equation}\label{eq:hyp_estim}
	\left|{R_{\rho_3}(\Psi_{\epsilon,\Phi})}\right|_ { W_{k, \infty}^{n, \infty}((0,1),(0,1)^d)}
	\leq C_1N^{n+k}, 
\end{equation}
where $n,k\in \{0,1\}$, such that $n+k=1$ or $n=k=1$,
$N\geq 1$ and $C_1$ depends on $d+1$ and $\mu$.
The first case  in the induction is fulfilled obviously when $\mu=1$ we can choose
$\Psi_{\epsilon,\Phi}=\Phi$ and the claim holds for any $\epsilon\in(0, 1/2)$. 
	
Now we show the second case of the induction, that is, let the claim holds
for some $\mu\in\mathbb{N}$ and we prove that it holds also  for $\mu+1$.

For this, let $\epsilon \in(0, 1/2)$ and let
$\Phi=((A_1,b_1),(A_2,b_2),\dots,(A_L,b_L))$ be a neural network with $d+1$-dimensional
input and $m$-dimensional output, where $m\leq \mu+1$, and with number of layers,
neurons and weights all bounded by $K$, where each $A_l$ is an
$N_l\times \sum_{k=0}^{l-1} N_k$ matrix, and $b_l\in\mathbb{R}^{N_l}$ for $l=1,\ldots L$.

We split the rest of the proof on two steps the first is the case when $m\leq \mu$
and the second deals with the case $m=\mu+1$.

\customlabel{step1}{\textbf{Step 1}}:
If $m\leq \mu$, then we use the induction hypothesis and get that there is a
constant $c_0=c_0(\mu)>0$ and a neural network $\Psi_{\epsilon,\Phi}$ with
$d+1$-dimensional input and one-dimensional output, and at most
$Kc_0$ layers, neurons and weights such that 
\begin{equation*}
\Vert{R_{\rho_3}(\Psi_{\epsilon,\Phi})
	-\prod_{l=1}^m [R_{\rho_3}(\Phi)]_l}\Vert_
	{ W_{k, \infty}^{n, \infty}((0,1),(0,1)^d)}
\leq \mu^{2-k-n}c_1^{k+n} N^{k+n} \epsilon
\leq (\mu+1)^{2-n- k}c_1^{n+k} N^{n+k}\epsilon
\end{equation*}
for $n, k\in\{0,1\}$ and $c_1=c_1(d+1,\mu)$. Moreover, 
\begin{equation*}
R_{\rho_3}(\Psi_{\epsilon,\Phi})(t, x)=0\quad \text{if}\quad
\prod_{l=1}^m [R_{\rho_3}(\Phi)]_l(t, x)=0,
\end{equation*}
for any $t\in(0,1)$ and $x\in(0,1)^d$.
Furthermore, for $n,k \in \{0,1\}$, we have
$\left|{R_{\rho_3}(\Psi_{\epsilon,\Phi})}\right|
		_{ W_{k,\infty}^{n, \infty}((0,1),(0,1)^d)} \leq C_1 N^{n+k}$,
where  $C_1=C_1(d+1,\mu)>0$.

\par

\customlabel{step2}{\textbf{Step 2}}:
Let $m=\mu+1$ and show the claim for constants
$\tilde c_0,\tilde c_1$ and~$\tilde c$ depending on $\mu+1$, possibly different
from the constants $c_0,c_1$ and $C_1$ from \ref{step1}, respectively.
	
We denote by $\Phi_\mu$ the neural network with $d+1$-dimensional input and
$\mu$-dimensional output which results from $\Phi$ by removing the last output
neuron and corresponding weights. In detail, we write 
\[
A_L=\left[
\begin{array}{c}
A_L^{(1,\mu)}
\\[1em]
a_L^{(\mu+1)}
\end{array}
	\right]\quad\text{and}\quad
b_L=
{\left[ \begin{array}{c}
b_L^{(1,\mu)}
\\[1ex]
b_L^{(\mu+1)}
\end{array} \right]},
\]
where $A_L^{(1,\mu)}$ is a $\mu\times\sum_{k=0}^{L-1}N_k$ matrix and
$a_L^{(\mu+1)}$ is a $1\times\sum_{k=0}^{L-1}N_k$ vector,
and $b_L^{(1,\mu)}\in\mathbb{R}^\mu$
and $b_L^{(\mu+1)}\in\mathbb{R}$. Now we set 
\[
\Phi_\mu:=\Big((A_1,b_1),(A_2,b_2),\dots,(A_{L-1},b_{L-1}),
\Big(A_L^{(1,\mu)},b_L^{(1,\mu)}\Big)\Big).
\]
Using the induction hypothesis and the constants $c_0,c_1$ and $C_1$ from \ref{step1},
we get that there is a neural network
$\Psi_{\epsilon,\Phi_\mu}=((A'_1,b'_1),(A'_2,b'_2),\dots,(A'_{L'},b'_{L'}))$ with
$d+1$-dimensional input and one-dimensional output, and at most
$Kc_0$ layers, neurons and weights such that
\begin{equation}\label{eq:induction_bound}
	\Vert{R_{\rho_3}(\Psi_{\epsilon,\Phi_\mu})-\prod_{l=1}^\mu
	[R_{\rho_3}(\Phi_\mu)]_l}\Vert_{ W_{k, \infty}^{n, \infty}((0,1),(0,1)^d)}
	\leq \mu^{2-k-n}c_1^{k+n} N^{k+n} \epsilon
\end{equation}
for $n, k\in\{0,1\}$. Moreover, 
\begin{equation}\label{eq:induction_zero}
	R_{\rho_3}(\Psi_{\epsilon,\Phi_\mu})(t, x)=0\quad \text{if}\quad
	\prod_{l=1}^\mu [R_{\rho_3}(\Phi_\mu)]_l(t, x)=0,
\end{equation}
for any $t \in (0,1)$ and $x\in(0,1)^d$.
Furthermore, we can assume that
$\left|{R_{\rho_3}(\Psi_{\epsilon,\Phi_\mu})}\right|
_{ W_{1,\infty}^{1, \infty}((0,1),(0,1)^d)}\leq C_1 N^{2}$,
and that the first $L(\Phi)-1$ layers of $\Psi_{\epsilon,\Phi_\mu}$ and $\Phi_\mu$
coincide and, thus, also the first $L(\Phi)-1$ layers of $\Psi_{\epsilon,\Phi_\mu}$
and $\Phi$, i.e.\ $A_l=A'_l$ for $l=1,\ldots,L(\Phi)-1$. 
	
Now, we add the formerly removed neuron with corresponding weights back to
the last layer of $\Psi_{\epsilon,\Phi_\mu}$. For the resulting network
\[
\widetilde \Psi_{\epsilon,\Phi}:=\left((A'_1,b'_1),(A'_2,b'_2),\dots,(A'_{L'-1},b'_{L'-1}),
\left(
{\left[ \begin{array}{c c}
\multicolumn{2}{c}{A'_{L'}}\\[1ex]
a^{(\mu+1)}_L & 0_{\mathbb{R}^{1,\sum_{k=L}^{L'}N_L'}}
		\end{array} \right]},
{\left[ \begin{array}{c}
b'_{L'}\\[1ex]
b_L^{(\mu+1)}
	\end{array} \right]}\right)\right)
\]
it holds that the first $L-1$ layers of $\widetilde \Psi_{\epsilon,\Phi}$ and $\Phi$
coincide, and $\widetilde \Psi_{\epsilon,\Phi}$ is a neural network with two-dimensional output.
Note that
\begin{align*}
&\Vert{\big[R_{\rho_3} \big({\widetilde\Psi_{\epsilon,\Phi}}\big)\big]_1}\Vert
_{{L_t^\infty L_x^{\infty}((0,1),(0,1)^d)}}
=\Vert{R_{\rho_3}(\Psi_{\epsilon,\Phi_\mu})}\Vert_
{{L_t^\infty L_x^{\infty}((0,1),(0,1)^d)}}
\\
&\hspace{5mm}\leq \Vert{R_{\rho_3}(\Psi_{\epsilon,\Phi_\mu})
	-\prod_{l=1}^\mu [R_{\rho_3}(\Phi_\mu)]_l}\Vert_
	{{L_t^\infty L_x^{\infty}((0,1),(0,1)^d)}}
		+\Vert{\prod_{l=1}^\mu [R_{\rho_3}(\Phi_\mu)]_l}\Vert_
			{{L_t^\infty L_x^{\infty}((0,1),(0,1)^d)}}
\\
&\hspace{5mm}\leq \mu^2 \epsilon + 1< \mu^2 +1,
\end{align*}
where we used \eqref{eq:induction_bound} for $n=k=0$, 
\eqref{eq:mult_variate_loc_mon_identity} and the properties of the partition of unity.
Additionally, we have 
\[
 \Vert{\big[R_{\rho_3} \big({\widetilde\Psi_{\epsilon,\Phi}}\big)\big]_2}
 \Vert_{{L_t^\infty L_x^{\infty}((0,1),(0,1)^d)}}
 =\Vert{[R_{\rho_3}(\Phi)]_{\mu+1}}\Vert_{{L_t^\infty L_x^{\infty}((0,1),(0,1)^d)}}\leq 1.
\]

\par

Now, we denote by $\times_r$ the network from
Remark~\ref{rem:recu_approximate_product}
with $r = \mu^2+1$ such that for any $\epsilon\in (0, 1/2)$,
we have
\[
	\Vert R_{{\rho_3}}(\times_r) (t, x)- tx \Vert_
	{W_{1, \infty}^{1, \infty}((-(\mu^2+1),\mu^2+1),(-(\mu^2+1),\mu^2+1))}
	< \epsilon
\]
previous estimate holds true since, in Remark~\ref{rem:recu_approximate_product},
$R_{{\rho_3}}(\times_r)$ present exactly the product $tx$ in $W_{1, \infty}^{1, \infty}$.
Moreover, we define 
\[
\Psi_{\epsilon,\Phi}:=\times_r\bullet\widetilde \Psi_{\epsilon,\Phi}.
\] 
Consequently, combining the induction hypothesis with
Remark~\ref{rem:recu_approximate_product}
and Remark \ref{remark:neural_sparse_concat},
 $\Psi_{\epsilon,\Phi}$ has $d+1$-dimensional input, one-dimensional
output and at most 
\(
K' + Kc_0 + K'Kc_0
\leq KC
\)
layers, number of neurons and weights, where  $K'=16$ is
the constant from Remark~\ref{rem:recu_approximate_product}
and $C=C(\mu)>0$ is a suitable constant.
Moreover, the first $L-1$ layers of $\Psi_{\epsilon,\Phi}$ and $\Phi$
coincide and for $n, k\in\{0,1\}$ the following approximation holds
\begin{align*}
&\left|{R_{\rho_3}(\Psi_{\epsilon,\Phi})-
	\prod_{l=1}^{\mu+1} [R_{\rho_3}(\Phi)]_l}\right|_{ W_{k, \infty}^{n,\infty}
		((0,1),(0,1)^d)}
\\
&\hspace{5mm} = \left|{R_{\rho_3}(\times_r)\circ
R_{\rho_3}\big({\widetilde\Psi_{\epsilon,\Phi}}\big)-
	[R_{\rho_3}(\Phi)]_{\mu+1}\cdot\prod_{l=1}^{\mu} [R_{\rho_3}(\Phi)]_l}\right|_
	{ W_{k, \infty}^{n,\infty}((0,1),(0,1)^d)}
\\
&\hspace{5mm} \leq \left|{R_{\rho_3}(\times_r)\circ
	(R_{\rho_3}(\Psi_{\epsilon,\Phi_\mu}),[R_{\rho_3}(\Phi)]_{\mu+1}) 
		-R_{\rho_3}(\Psi_{\epsilon,\Phi_\mu})\cdot[R_{\rho_3}(\Phi)]_{\mu+1}}\right|_
		{ W_{k, \infty}^{n,\infty}((0,1),(0,1)^d)}
\\
&\hspace{5mm} \hspace{5mm}+\left|{[R_{\rho_3}(\Phi)]_{\mu+1}\cdot
	\Big(R_{\rho_3}(\Psi_{\epsilon,\Phi_\mu})-\prod_{l=1}^{\mu}
			[R_{\rho_3}(\Phi)]_l\Big)}\right|_
			{ W_{k, \infty}^{n,\infty}((0,1),(0,1)^d)}.
			\numberthis\label{eq:induction_add_zero}
\end{align*}
Let $n=k=0$, then, for the first term in inequality \eqref{eq:induction_add_zero}, using
Remark \ref{rem:recu_approximate_product} we obtain
\begin{align}\label{eq:compo_estim}
&\Vert{R_{\rho_3}(\times_r)\circ(R_{\rho_3}(\Psi_{\epsilon,\Phi_\mu}),
	[R_{\rho_3}(\Phi)]_{\mu+1}) -R_{\rho_3}(\Psi_{\epsilon,\Phi_\mu})
		\cdot[R_{\rho_3}(\Phi)]_{\mu+1}}\Vert_{L_t^\infty L_x^{\infty}((0,1),(0,1)^d)}
\\
&\hspace{5mm}\leq
\Vert{R_{\rho_3}(\times_r)(t,x)-t\cdot x}\Vert_
{L_t^\infty L_x^{\infty}(({-(\mu^2+1)},{\mu^2+1}), ({-(\mu^2+1)},{\mu^2+1}))}
	\leq \epsilon.\numberthis\label{eq:1term_0k}
\end{align}
Next, for $n,k \in \{	0,1\}$ such that $n+k=1$ and apply the chain rule from
Lemma \ref{cor:composition_norm} to \eqref{eq:compo_estim}.
For this, let $\hat C=\hat C(d+1)$ be the constant from
Lemma \ref{cor:composition_norm} (for $p=d+1$ and $m=2$).
Using the induction hypothesis together with the fact that
$\left|{[R_{\rho_3}{(\Phi)}]_{\mu+1}}\right|_
{ W_{k, \infty}^{n, \infty}((0,1)^d)}\leq N^{n+k}$,
we get
\begin{align*}
&\left|{R_{\rho_3}(\times_r)\circ(R_{\rho_3}(\Psi_{\epsilon,\Phi_\mu}),
	[R_{\rho_3}(\Phi)]_{\mu+1}) -R_{\rho_3}(\Psi_{\epsilon,\Phi_\mu})\cdot
		[R_{\rho_3}(\Phi)]_{\mu+1}}\right|_
		{ W_{k, \infty}^{n, \infty}((0,1), (0,1)^d)}
\\
&\hspace{5mm}\leq \hat C\cdot
\left|{R_{\rho_3}(\times_r)(t, x) -t\cdot x}\right|_{ W_{k, \infty}^{n, \infty}
(({-(\mu^2+1)},{\mu^2+1}), ({-(\mu^2+1)},{\mu^2+1}))}
	\left|{R_{\rho_3}\big({\widetilde\Psi_{\epsilon,\Phi}}\big)}\right|_
	{ W_{k,\infty}^{n,\infty}((0,1),(0,1)^d;\mathbb{R}^2)}
\\
&\hspace{5mm}\leq \hat C\cdot\epsilon \max\{C_1 N,N\}=C_1' \epsilon N,
\numberthis\label{eq:1term_1k}
\end{align*}
where $C_1'=C_1'(d+1)>0$.
Now, in similar way,  we treat the case where $n=k=1$	for the same quantity
in the previous inequality.
In view of the second result of Lemma \ref{cor:composition_norm} for some constant
$C'=C'(d+1)>0$, we get
\begin{multline*}
\left|{R_{\rho_3}(\times_r)\circ(R_{\rho_3}(\Psi_{\epsilon,\Phi_\mu}),
	[R_{\rho_3}(\Phi)]_{\mu+1}) -R_{\rho_3}(\Psi_{\epsilon,\Phi_\mu})\cdot
		[R_{\rho_3}(\Phi)]_{\mu+1}}\right|_
		{ W_{1, \infty}^{1, \infty}((0,1), (0,1)^d)}
\\
\hspace{5mm}\leq  C'\cdot\max\left(
\left|{R_{\rho_3}(\times_r)(t, x) -t\cdot x}\right|_{ W_{1, \infty}^{0, \infty}
(({-(\mu^2+1)},{\mu^2+1}), ({-(\mu^2+1)},{\mu^2+1}))}
	\left|{R_{\rho_3}\big({\widetilde\Psi_{\epsilon,\Phi}}\big)}\right|_
	{ W_{1,\infty}^{1,\infty}((0,1),(0,1)^d;\mathbb{R}^2)},\right.
\\
\hspace{10mm}\left.\left|{R_{\rho_3}(\times_r)(t, x) -t\cdot x}\right|_
{ W_{1, \infty}^{1, \infty}
(({-(\mu^2+1)},{\mu^2+1}), ({-(\mu^2+1)},{\mu^2+1}))}
	\left|{R_{\rho_3}\big({\widetilde\Psi_{\epsilon,\Phi}}\big)}\right|_
	{ W_{0,\infty}^{1,\infty}((0,1),(0,1)^d;\mathbb{R}^2)}^2\right)	
\\
\hspace{5mm}\leq
C'\cdot\max(\epsilon \max\{C_1 N^2,N^2\},
\epsilon \max\{C_1^2 N^2,N^2\})
=C_2'  N^2\epsilon,\numberthis\label{eq:1term_2k}
\end{multline*}
where $C_2'=C_2'(d+1)>0$.

\par
	
It remains to estimate the second term of \eqref{eq:induction_add_zero} for $n= k=0$.
Thus, under the induction hypothesis (for $n=k=0$) and get
\begin{align*}
&\Vert{[R_{\rho_3}(\Phi)]_{\mu+1}\cdot\Big(R_{\rho_3}(\Psi_{\epsilon,\Phi_\mu})
	-\prod_{l=1}^{\mu} [R_{\rho_3}(\Phi)]_l\Big)}\Vert_{ L^\infty((0,1)^d)}
\\
&\hspace{5mm} \leq \Vert{[R_{\rho_3}(\Phi)]_{\mu+1}}\Vert_{ L^\infty((0,1)^d)}
	\cdot\Vert R_{\rho_3}(\Psi_{\epsilon,\Phi_\mu})-\prod_{l=1}^{\mu}
		[R_{\rho_3}(\Phi)]_l)\Vert_{L^{\infty}((0,1)^d)}\leq 1\cdot \mu^2 \epsilon.
		\numberthis\label{eq:2term_0k}
\end{align*}

\par

The case $n, k\in\{0,1\}$ such that $n+k=1$ can be obtained by applying the product rule from
Lemma~\ref{lemma:product_rule_bound_p} together with
$\Vert{[R_{\rho_3}{(\Phi)}]_{\mu+1}}\Vert_{L^{\infty}}\leq 1$, indeed
\begin{align*}
&\left|{[R_{\rho_3}(\Phi)]_{\mu+1}\cdot\Big(R_{\rho_3}(\Psi_{\epsilon,\Phi_\mu})
	-\prod_{l=1}^{\mu} [R_{\rho_3}(\Phi)]_l\Big)}\right|_
	{ W_{k,\infty}^{n, \infty}((0,1),(0,1)^d)}
\\
&\hspace{5mm} \leq \left|{[R_{\rho_3}(\Phi)]_{\mu+1}}\right|
	_{ W_{k,\infty}^{n, \infty}((0,1),(0,1)^d)}
	\cdot\Vert{R_{\rho_3}(\Psi_{\epsilon,\Phi_\mu})
		-\prod_{l=1}^{\mu} [R_{\rho_3}(\Phi)]_l}\Vert_{{L^{\infty}((0,1),(0,1)^d)}}
\\
&\hspace{5mm}\hspace{5mm} + \Vert{[R_{\rho_3}(\Phi)]_{\mu+1}}
	\Vert_{{L^{\infty}((0,1),(0,1)^d)}}
	\cdot\left|{R_{\rho_3}(\Psi_{\epsilon,\Phi_m})-\prod_{l=1}^{\mu}
			[R_{\rho_3}(\Phi)]_l}\right|_{ W_{k,\infty}^{n, \infty}((0,1),(0,1)^d)}
\\
&\hspace{5mm}\leq N\cdot \mu^2\epsilon+1\cdot \mu c_1 N\epsilon
	=c_1'N\epsilon,\numberthis\label{eq:2term_1k}
\end{align*}
where we used the induction hypothesis for $k+n=1$, and $c_1'=c_1'(d+1,\mu)>0$.

\par

The last case is $n = k=1$ can be concluded, in a similar way as the previous case,
by applying the product rule from
Lemma~\ref{lemma:product_rule_bound_p} together with
$\Vert{[R_{\rho_3}{(\Phi)}]_{\mu+1}}\Vert_{L^{\infty}}\leq 1$, indeed
\begin{align*}
&\left|{[R_{\rho_3}(\Phi)]_{\mu+1}\cdot\Big(R_{\rho_3}(\Psi_{\epsilon,\Phi_\mu})
	-\prod_{l=1}^{\mu} [R_{\rho_3}(\Phi)]_l\Big)}\right|_
	{ W_{1,\infty}^{1, \infty}((0,1),(0,1)^d)}
\\
&\hspace{5mm} \leq\left\vert[R_{\rho_3}(\Phi)]_{\mu+1}\right\vert_
{W_{1, \infty}^{1, \infty}((0,1),(0,1)^d)}
\Vert R_{\rho_3}(\Psi_{\epsilon,\Phi_\mu})-\prod_{l=1}^{\mu}
 [R_{\rho_3}(\Phi)]_l \Vert_{L^\infty((0,1),(0,1)^d)}
\\
&\hspace{5mm}+\left\vert[R_{\rho_3}(\Phi)]_{\mu+1}\right\vert_
{W_{1, \infty}^{0, \infty}((0,1),(0,1)^d)}
\left\vert R_{\rho_3}(\Psi_{\epsilon,\Phi_\mu})-
	\prod_{l=1}^{\mu} [R_{\rho_3}(\Phi)]_l \right\vert_
	{W_{0, \infty}^{1,\infty}((0,1),(0,1)^d)}
\\
&\hspace{5mm}+\left\vert[R_{\rho_3}(\Phi)]_{\mu+1}\right\vert_
{W_{0, \infty}^{1, \infty}((0,1),(0,1)^d)}
\left\vert R_{\rho_3}(\Psi_{\epsilon,\Phi_\mu}) -
	\prod_{l=1}^{\mu} [R_{\rho_3}(\Phi)]_l \right\vert_
	{W_{1, \infty}^{0, \infty}((0,1),(0,1)^d)}
\\
&\hspace{5mm}+ \Vert [R_{\rho_3}(\Phi)]_{\mu+1}\Vert_
{L^\infty((0,1),(0,1)^d)}
\left\vert R_{\rho_3}(\Psi_{\epsilon,\Phi_\mu})
	-\prod_{l=1}^{\mu} [R_{\rho_3}(\Phi)]_l\right\vert_
	{W_{1,\infty}^{1, \infty}((0,1),(0,1)^d)}
\\
&\hspace{5mm}\leq N^2\cdot \mu^2\epsilon+
N\cdot \mu c_1N\epsilon+
N\cdot \mu c_1 N\epsilon+
1\cdot  c_1^2 N^2\epsilon
	=c_2'N^2\epsilon,\numberthis\label{eq:2term_2k}
\end{align*}
where we used the induction hypothesis for $k+n=1$,
$n=k=1$, and $c_2'=c_2'(d+1,\mu)>0$.

\par

Then, a combination of \eqref{eq:induction_add_zero} with \eqref{eq:1term_0k},
\eqref{eq:2term_0k} and \eqref{eq:2term_0k} yields 
\begin{equation}\label{eq:final_s_0}
\Vert{R_{\rho_3}(\Psi_{\epsilon,\Phi})-\prod_{l=1}^{\mu+1}
	[R_{\rho_3}(\Phi)]_l}\Vert_{ L^\infty((0,1),(0,1)^d)}\leq \epsilon
		+\mu^2\cdot \epsilon=(\mu^2+1)\cdot\epsilon.
\end{equation}
Similarly a combination of \eqref{eq:induction_add_zero}
with \eqref{eq:1term_1k} and \eqref{eq:2term_1k}, for $n,k\in\{0,1\}$
such that $n+k=1$, we get
\begin{equation*}
\left|{R_{\rho_3}(\Psi_{\epsilon,\Phi})
	-\prod_{l=1}^{\mu+1} [R_{\rho_3}(\Phi)]_l}\right|_{ W_{k, \infty}^{n, \infty}((0,1),(0,1)^d)}
		\leq (C_1'+c_1')\cdot N\cdot \epsilon=c_1'' N\epsilon,
\end{equation*}
where $c_1''=c_1''(d+1,\mu)>0$.

Moreover, for the case where $n=k=1$ we combine \eqref{eq:induction_add_zero}
with \eqref{eq:1term_2k} and \eqref{eq:2term_2k}, we get
\begin{equation*}
\left|{R_{\rho_3}(\Psi_{\epsilon,\Phi})
	-\prod_{l=1}^{\mu+1} [R_{\rho_3}(\Phi)]_l}\right|_{ W_{1, \infty}^{1, \infty}((0,1),(0,1)^d)}
		\leq (C_2'+c_2')\cdot N^2\cdot \epsilon=c_2'' N^2\epsilon,
\end{equation*}
where $c_2''=c_2''(d+1,\mu)>0$.
In view of the three previous estimates, we get
\begin{equation}\label{eq:final_s_1}
\Vert{R_{\rho_3}(\Psi_{\epsilon,\Phi})
	-\prod_{l=1}^{\mu+1} [R_{\rho_3}(\Phi)]_l}\Vert_{ W_{k, \infty}^{n, \infty}((0,1),(0,1)^d)}
		\leq c_1''' N^{n+k}\epsilon,
\end{equation}
for a suitable constant $c_1'''=c_1'''(d+1, \mu, k,n)>0$. 	

\par

Finally, we show \eqref{eq:zero_multiplication} for $\mu+1$, this is follow by
similar argument as in \cite[Lemma C.5]{GuhKutPet}, for the sake of completeness we show it.
Thus, let $[R_{\rho_3}(\Phi)]_l(t, x)=0$ for some $l\in\{1,\ldots,\mu+1\}$,
$t \in (0,1)$ and $x\in(0,1)^d$. In the case where $l\leq \mu$,
\eqref{eq:induction_zero} implies that 
\[
	\big[R_{\rho_3}\big({\widetilde\Psi_{\epsilon,\Phi}}\big)\big]_1(t,x)
	=R_{\rho_3}(\Psi_{\epsilon,\Phi_\mu})(t, x)=0.
\]
Moreover, if $l=\mu+1$, then 
\[
	\big[R_{\rho_3}\big({\widetilde\Psi_{\epsilon,\Phi}}\big)\big]_2(t, x)
	=[R_{\rho_3}(\Phi)]_{\mu+1}(t, x)=0.
\]
Hence, by application of Remark \ref{rem:recu_approximate_product},
we have 
\[
R_{\rho_3}(\Psi_{\epsilon,\Phi})(x)
	=R_{\rho_3}(\times_r)\left(\big[R_{\rho_3}\big({\widetilde\Psi
		_{\epsilon,\Phi}}\big)\big]_1(x),
			\big[R_{\rho_3}\big({\widetilde\Psi_{\epsilon,\Phi}}\big)\big]_2(x)\right)=0.
\]

\par

Before concluding the proof, we need to show \eqref{eq:hyp_estim}.
If $n+k=1$, we use Lemma \ref{cor:composition_norm},
Remark \ref{rem:recu_approximate_product} and similar argument as in \eqref{eq:1term_1k},
then we get
\begin{align*}
\left|{R_{\rho_3}(\Psi_{\epsilon,\Phi})}\right|&_{ W_{k, \infty}^{n, \infty}((0,1),(0,1)^d)}=
	\left|{R_{\rho_3}(\times_r)\circ R_{\rho_3}
	\big({\widetilde\Psi_{\epsilon,\Phi}}\big)}\right|_{ W_{k, \infty}^{n, \infty}((0,1),(0,1)^d)}
\\
&\leq \hat C\cdot \left|{R_{\rho_3}(\times_r)}\right|_
	{W_{k, \infty}^{n,\infty}(({-(\mu^2+1)},{\mu^2+1}),({-(\mu^2+1)},{\mu^2+1}))}
	\cdot\left|{R_{\rho_3}\big({\widetilde\Psi_{\epsilon,\Phi}}\big)}\right|
		_{ W_{k, \infty}^{n,\infty}((0,1),(0,1)^d;\mathbb{R}^2)}
\\
&\leq \hat C\cdot(\mu^2+1)\cdot\max\left\{C_1 N,N\right\}= C_1''\cdot N,
\end{align*}
where $C_1''=C_1''(d+1, \mu)>0$ is a suitable constant.

If $n=k=1$, in view of Lemma \ref{cor:composition_norm},
Remark \ref{rem:recu_approximate_product} and similar argument as in \eqref{eq:1term_2k},
we have
\begin{multline*}
\left|{R_{\rho_3}(\Psi_{\epsilon,\Phi})}\right|_{ W_{1, \infty}^{1, \infty}((0,1),(0,1)^d)}=
	\left|{R_{\rho_3}(\times_r)\circ R_{\rho_3}
	\big({\widetilde\Psi_{\epsilon,\Phi}}\big)}\right|_{ W_{1, \infty}^{1, \infty}((0,1),(0,1)^d)}
\\
\leq  C'\cdot\max\left(
\left|{R_{\rho_3}(\times_r)}\right|_{ W_{1, \infty}^{0, \infty}
(({-(\mu^2+1)},{\mu^2+1}), ({-(\mu^2+1)},{\mu^2+1}))}
	\left|{R_{\rho_3}\big({\widetilde\Psi_{\epsilon,\Phi}}\big)}\right|_
	{ W_{1,\infty}^{1,\infty}((0,1),(0,1)^d;\mathbb{R}^2)},\right.
\\
\left.\left|{R_{\rho_3}(\times_r)}\right|_{ W_{1, \infty}^{1, \infty}
(({-(\mu^2+1)},{\mu^2+1}), ({-(\mu^2+1)},{\mu^2+1}))}
	\left|{R_{\rho_3}\big({\widetilde\Psi_{\epsilon,\Phi}}\big)}\right|_
	{ W_{0,\infty}^{1,\infty}((0,1),(0,1)^d;\mathbb{R}^2)}^2\right)	
\\
\leq  C'\cdot\max((\mu^2+1) \max\{C_1 N^2,N^2\},
 \max\{C_1^2 N^2,N^2\})
 \leq C_2'' N^2,
\end{multline*}
where $C_2''=C_2''(d+1, \mu)>0$ is a suitable constant.

To conclude, we take the maximum of the constants derived in
\ref{step1} and \ref{step2}.
\end{proof}

\par

Next result  is the final step toward the main theorem of our paper. Mainly, we show an upper bound
error in Sobolev time-space for a sum of localized polynomials with deep neural network.
Thus, we get approximation and regularity information about the network.

\par

\begin{lemma}\label{lemma:network_polynomial_approximation}
Let $d, m, N \in \mathbb{N}, 1 \leq p, q \leq \infty$,
such that $m \geq n+k+ 1$,
$n, k \in \{0,1\}$ and
$\Psi = \left\{\phi_{\mu}: \mu \in\{0, \ldots, N\}^{d+1}\right\}$
be the partition of unity from Lemma \ref{lemma:partition_of_unity}.
Then,
there are constants $C_{1}=C_{1}(m, d+1, p, n,k)>0$ and
$C_{2}=C_{2}(m, d+1), C_{3}=C_{3}(m, d+1)>0$
with the following properties:
For any $\epsilon \in(0,1 / 2),$ there is a neural network architecture
$\mathcal{A}_{\epsilon}=\mathcal{A}_{\epsilon}(d+1, m, N, \epsilon)$
with $d+1$-dimensional input and one-dimensional output,
at most $C_{2}$ layers and $C_{3}(N+1)^{2(d+1)}$ neurons and weights
such that the following holds: Let $u \in W_{m, p}^{m, p}\left((0,1),(0,1)^{d}\right)$
and $p_{u,\mu}(t, x)=\sum_{\kappa +|\alpha| \leq m-1}
c_{\mu,\kappa, \alpha}t ^\kappa x^{\alpha}$ for
$\mu \in\{0, \ldots, N\}^{d+1}$ be the polynomials from
Lemma~\ref{lemma:polynomial_approximation}, then there is a neural network
$\Phi_{P, \epsilon}$ that has architecture $\mathcal{A}_{\epsilon}$
such that
\begin{equation}\label{eq:loc_pol_estim}
\left\|\sum_{\mu \in\{0, \ldots, N\}^{d+1}} \phi_{\mu} p_{u,\mu} -
R_{\rho_3}\left(\Phi_{P, \epsilon}\right)\right\|_{W_{k,q}^{n, p}\left((0,1),(0,1)^{d}\right)}
\leq C_{1} N^{n+ k+ 1} \epsilon\|u\|_{W_{m,p}^{m, p}\left((0,1),(0,1)^{d}\right)}
\end{equation}
\end{lemma}

\par

\begin{proof}
	We divide the proof in three steps.
	
	\par
	
	\textbf{Step 1 (Approximating localized monomials $\phi_\mu(t, x) \, t^\kappa x^\alpha$ ):}
	Let $\kappa + |\alpha| \leq  m-1$ and $\mu \in\{0, \ldots, N\}^{d+1}$.
	Since we can get $x$ out of the ReCU realization of $\phi_x$ for any $x\in [0, 1]$, where
	$\phi_x = \left( (A_1, b_1),(A_2, b_2)\right)$ and 
	\[A_1 = \begin{pmatrix}
		1\\
		-1\\
		1		
	   \end{pmatrix}, 
	   b_1 = \begin{pmatrix}
		1\\
		1\\
		0
	\end{pmatrix},
	A_2 = \frac 1{6}\begin{pmatrix}
		1\\
		-1\\
		-2
	\end{pmatrix} \text{ and } b_2=0.
	\]
Here $\phi_x$ is a two layer ReCU network with one-dimensional input and one-dimensional output,
$8$ nonzero weights and $5$ neurons.
Then, we can construct a neural network $\Phi_{\kappa, \alpha}$with $d+1$-dimensional input
and $\kappa + |\alpha|$-dimensional output, with two layer and at most
$8(m-1)$ nonzero weights and $ d+1+4(m-1)$ neurons
such that
$$
t^\kappa x^{\alpha}=\prod_{l=1}^{\kappa +|\alpha|}
\left[R_{\rho_3}\left(\Phi_{\kappa, \alpha}\right)\right]_{l}(t, x)
\quad \text { for any  } (t,x) \in (0,1)\times(0,1)^{d}
$$
and
\begin{equation}\label{eq:phi_poly_bound}
	\left\|\left[R_{\rho_3}
	\left(\Phi_{\kappa, \alpha}\right)\right]_{l}\right\|_
	{W_{k, \infty}^{n, \infty}\left((0,1), (0,1)^{d}\right)}
	\leq 2
	\quad \text{ for any }
	l=1, \ldots,\kappa +|\alpha| \quad \text{ and }\quad n, k \in\{0,1\}.
\end{equation}

From Lemma \ref{lemma:partition_of_unity}, we use the neural network 
$\Phi_{\mu}$ and the constants $C, c \geq 1$ to define the network
$\Phi_{ \kappa, \alpha}^\mu$ as the parallelization of $\Phi_{\mu}$ and
$\Phi_{\kappa, \alpha}$ (see Definition \ref{def:parallel_net}), that is,
$$
\Phi_{ \kappa, \alpha}^\mu:= P\left(\Phi_{\mu}, \Phi_{\kappa, \alpha}\right).
$$
Then $\Phi_{ \kappa, \alpha}^\mu$ has at most $3 \leq K_{0}$ layers, $C (d+1)+8(m-1) \leq K_{0}$
nonzero weights, and $C(d+1)+ 4(m-1) \leq K_{0}$ neurons for a suitable constant
$K_{0}=K_{0}(m, d+1) \in \mathbb{N}$, and
$\prod_{l=1}^{\kappa +|\alpha|+d+1}\left[R_{\rho_3}
\left(\Phi_{ \kappa, \alpha}^\mu\right)\right]_{l}(t, x)=\phi_{\mu}(t, x)t^\kappa x^{\alpha}$
for all $(t,x) \in(0,1)\times(0,1)^{d} .$ Moreover, as a consequence of
Lemma \ref{lemma:partition_of_unity} together with \eqref{eq:phi_poly_bound},
we have 
$$
\left\|\left[R_{\rho_3}
\left(\Phi_{ \kappa, \alpha}^\mu\right)\right]_{l}\right\|_
{W_{k, \infty}^{n, \infty}\left((0,1),(0,1)^{d}\right)} \
\leq(c N)^{k+n},\; \text{for any } l=1, \ldots,\kappa + |\alpha|+ d+ 1
\; \text{ and }\; k, n \in\{0,1\}.
$$

Let $\Psi_{\epsilon, \Phi_{ \kappa, \alpha}^\mu}$ be the neural network from
Lemma \ref{lemma:network_multiplikation}
(with $ \Phi_{ \kappa, \alpha}^\mu$ instead of $\Phi$, $\mu= m+d \in \mathbb{N}$,
$K=K_{0} \in \mathbb{N}$ and $c N$ instead of $N$) for $\mu \in\{0, \ldots, N\}^{d+1}$ and
$\alpha \in \mathbb{N}_{0}^{d}$, $\kappa +|\alpha| \leq m-1$.
There exists a constant $C_{1}=C_{1}(m, d+1) \geq 1$ such that
$\Psi_{\epsilon, \Phi_{ \kappa, \alpha}^\mu}$
has at most $C_{1}$ layers, number of neurons, and weights.
Moreover,
\begin{equation}\label{eq:approx_mon_approx}
\left\|\phi_{\mu}(t, x)t^\kappa x^{\alpha}-R_{\rho_3}
\left(\Psi_{\epsilon, \Phi_{ \kappa, \alpha}^\mu}\right)(t, x)\right\|_
{W_{k, \infty}^{n, \infty}
\left((0,1),(0,1)^{d}\right)} \leq c^{\prime} N^{n+ k} \epsilon
\end{equation}
for a constant $c^{\prime}=c^{\prime}(m, d+1, k,n)>0$ and $n, k \in\{0,1\}$, and
\begin{equation}\label{eq:support_mon_approx}
R_{\rho_3}\left(\Psi_{\epsilon, \Phi_{ \kappa, \alpha}^\mu}\right)(t, x)=0 \quad
\text { if } \phi_{\mu}(t, x) t^\kappa x^{\alpha}=0 \quad \text { for all }
(t,x) \in(0,1)\times(0,1)^{d}.
\end{equation}
Thus, $R_{\rho_3}\left(\Psi_{\epsilon, \Phi_{ \kappa, \alpha}^\mu}\right)$ approximates
the localized mixed monomials $\phi_{\mu}(t, x) t^\kappa x^{\alpha}$
for all $(t,x) \in(0,1)\times(0,1)^{d}$, $\kappa + |\alpha|\leq m-1$,
in $W_{k, \infty}^{n, \infty}$ such that $n, k \in \{0,1 \}$.

\par

\textbf{Step 2
(Constructing an architecture capable of approximating sums of mixed localized polynomials):}
We set
$$
M:=\left|\left\{(\mu,\kappa, \alpha): \mu \in\{0, \ldots, N\}^{d+1},
\alpha \in \mathbb{N}_{0}^{d},\kappa +|\alpha| \leq m-1\right\}\right|
$$
and define the matrix $A_{\text {sum }} \in \mathbb{R}^{1 \times M}$ by
$$
A_{\text {sum }}:=\left[c_{\mu, \kappa, \alpha}: \mu \in\{0, \ldots, N\}^{d+1},
\alpha \in \mathbb{N}_{0}^{d},\kappa +|\alpha| \leq m-1\right]
$$
and the neural network
$\Phi_{\text {sum }}:=\left(\left(A_{\text {sum }}, 0\right)\right)$.
Finally, we set
$$
\Phi_{P, \epsilon}:=\Phi_{\mathrm{sum}} \bullet P
\left(\Psi_{\epsilon, \Phi_{ \kappa, \alpha}^\mu}: \mu \in\{0, \ldots, N\}^{d+1},
\alpha \in \mathbb{N}_{0}^{d},\kappa +|\alpha| \leq m-1\right).
$$
Then, there are constants $C_{2}=C_{2}(m, d+1), C_{3}=C_{3}(m, d+1)>0$
such that $\Phi_{P, \epsilon}$
is a neural network with $d+1$-dimensional input and one-dimensional output,
with at most $1+C_{1} \leq C_{2} $ layers,
$
M+M C_{1} +M^2 C_{1}
\leq 3 M^2 C_{1} 
\leq C_{3}(N+1)^{2(d+1)} 
$
nonzero weights and $M+M C_{1}\leq  2M C_{1} 
\leq C'_{3}(N+1)^{d+1}$  neurons, and
$$
R_{\rho_3}\left(\Phi_{P, \epsilon}\right)=
\sum_{\mu \in\{0, \ldots, N\}^{d+1}}\sum_{\kappa +|\alpha|
\leq m-1}c_{\mu, \kappa,\alpha}R_{\rho_3}\left( \Psi_{\epsilon, \Phi_{ \kappa, \alpha}^\mu}\right).
$$
Moreover, the network $\Phi_{P, \epsilon}$ depends only on $p_{u, \mu}$
(and thus on $u$) via the coefficients $c_{\mu,\kappa, \alpha}$.
Now, it is easy to see that there exists a neural network architecture
$\mathcal{A}_{\epsilon}=\mathcal{A}_{\epsilon}(d+1, m, N, \epsilon)$
with $L\left(\mathcal{A}_{\epsilon}\right) \leq C_{2}$
layers and number of neurons and weights bounded by
$C_{3}(N+1)^{2(d+1)}$ such that $\Phi_{P, \epsilon}$
has architecture $\mathcal{A}_{\epsilon}$ for every of choice of coefficients
$c_{\mu, \kappa, \alpha},$ and hence for every choice of $u$.

\textbf{Step 3
(Estimating the approximation error in $\|\cdot\|_{\boldsymbol{W}_{k, q}^{n, p}},
\boldsymbol{n, k} \in\{\mathbf{0}, \mathbf{1}\}$).}
Let 
$$
\Omega_{\mu, N}:=B_{\frac{1}{N},\|-\|_{\infty}}\left(\frac{\mu}{N}\right)
$$
such that $\mu \in\{0, \ldots, N\}^{d+1}$.
Moreover, for $n, k \in\{0,1\}$
\begin{multline*}
\left\|\sum_{\mu \in\{0, \ldots, N\}^{d+1}} \phi_{\mu}(t, x) p_{u, \mu}(t, x)-
R_{\rho_3}\left(\Phi_{P, \epsilon}\right)(t, x)\right\|_
{W_{k, q}^{n, p}\left((0,1),(0,1)^{d}\right)}
\\
=\left\|\sum_{\mu \in\{0, \ldots, N\}^{d+1}}
\sum_{\kappa + |\alpha| \leq m-1} c_{\mu,\kappa, \alpha}\left(\phi_{\mu}(t, x)
t^\kappa x^{\alpha} -R_{\rho_3}\left(\Psi_{\epsilon,\Phi_{\kappa, \alpha}^\mu }\right)(t, x)\right)
\right\|_{W_{k, q}^{n, p}\left((0,1),(0,1)^{d}\right)}
\\
\leq \sum_{\tilde{\mu} \in\{0, \ldots, N\}^{d+1}}
\left\| \sum_{\mu \in\{0, \ldots, N\}^{d+1}}\sum_{\kappa +|\alpha| \leq m-1}
c_{\mu,\kappa, \alpha}\left(\phi_{\mu}(t, x)t ^\kappa x^{\alpha}- R_{\rho_3}
\left(\Psi_{\epsilon,\Phi_{\kappa, \alpha}^\mu }\right)(t, x)\right)\right\|_
{W_{k, q}^{n, p}\left(\Omega_{\tilde{\mu}, N} \cap\left( (0,1)\times(0,1)^{d}\right) \right)}
\numberthis\label{eq:estimate_outer_sum}
\end{multline*}
last inequality holds true since
$(0,1)^{d+1} \subset \bigcup_{\widetilde{\mu} \in\{0, \ldots, N\}^{d+1}}
\Omega_{\widetilde{\mu}, N}$.
Therefore, using Lemma \ref{lemma:polynomial_approximation}, 
such that $\widetilde{\mu} \in \{0, \ldots, N\}^{d+1}$,
we get
\begin{multline*}
\left\| \sum_{\mu \in\{0, \ldots, N\}^{d+1}}\sum_{\kappa +|\alpha| \leq m-1}
c_{\mu,\kappa, \alpha}\left(\phi_{\mu}(t, x)t ^\kappa x^{\alpha}- R_{\rho_3}
\left(\Psi_{\epsilon,\Phi_{\kappa, \alpha}^\mu }\right)(t, x)\right)\right\|_
{W_{k, q}^{n, p}\left(\Omega_{\widetilde{\mu}, N} \cap(0,1)^{d+1} \right)}
\\
 \leq \sum_{\mu \in\{0, \ldots, N\}^{d+1}}\sum_{\kappa +|\alpha| \leq m-1}
c_1 N^{\frac {d+1}{p}}\left\| U \right\|_{W_{m, p}^{m, p}
\left(\Omega_{{\mu}, N}\right)}
\left\|\phi_{\mu}(t, x)t ^\kappa x^{\alpha}- R_{\rho_3}
\left(\Psi_{\epsilon,\Phi_{\kappa, \alpha}^\mu }\right)(t, x)\right\|_{W_{k, q}^{n, p}
\left(\Omega_{\widetilde{\mu}, N} \cap(0,1)^{d+1} \right)}
\numberthis\label{eq:approximate_loc_taylor_with_network},
\end{multline*}
where $U$ is an extension of $u$
and $c_{1}=c_{1}(m,d+1)>0$ is a constant. Next, note that
\begin{multline*}
\left\|\phi_{\mu}(t,x)t^\kappa x^{\alpha}-R_{\rho_3}
\left(\Psi_{\epsilon,\Phi_{\kappa, \alpha}^\mu }\right)(t, x)\right\|_{W_{k, q}^{n, p}
\left(\Omega_{\widetilde{\mu}, N} \cap(0,1)^{d+1}\right)}
\leq
\lambda \left(\Omega_{\widetilde{\mu}, N} \cap(0,1)
	\times \left(\{0\}_{\mathbb{R}^d}\right)\right)^{1 / q}
\\
\quad\times\lambda\left(\Omega_{\widetilde{\mu}, N} \cap
	 \left(\{0\}_{\mathbb{R}}\times(0,1)^{d}\right)\right)^{1 / p}
\left\|\phi_{\mu}(t,x)t^\kappa x^{\alpha}-R_{\rho_3}
\left(\Psi_{\epsilon,\Phi_{\kappa, \alpha}^\mu }\right)(t, x)\right\|_{W_{k, \infty}^{n, \infty}
\left(\Omega_{\widetilde{\mu}, N} \cap(0,1)^{d+1}\right)}
\\
\quad \leq c_{3}
\left(\frac{1}{N}\right)^{1 / q + d / p}
\left\|\phi_{\mu}(t,x)t^\kappa x^{\alpha}-R_{\rho_3}
\left(\Psi_{\epsilon,\Phi_{\kappa, \alpha}^\mu }\right)(t, x)\right\|_{W_{k, \infty}^{n, \infty}
\left(\Omega_{\widetilde{\mu}, N} \cap(0,1)^{d+1}\right)}
\end{multline*}
where $\lambda$ denotes the Lebesgue measure and $c_{3}=c_{3}(d+1, p, q)>0$
is a constant.A direct combination of \eqref{eq:approximate_loc_taylor_with_network}
with the last estimate, such that $N^{1/p -1/q}\leq N$, yields

\begin{multline*}
\left\|\sum_{\mu \in\{0, \ldots, N\}^{d+1}}\sum_{\kappa +|\alpha| \leq m-1}
c_{m, \alpha}\left(\phi_{\mu}(t,x)t^\kappa x^{\alpha}-R_{\rho_3}
\left(\Psi_{\epsilon,\Phi_{\kappa, \alpha}^\mu }\right)(t, x)\right)\right\|_{W_{k, q}^{n, p}
\left(\Omega_{\widetilde{\mu}, N} \cap(0,1)^{d+1}\right)}
\\
\leq c_{4}N \sum_{\mu \in\{0, \ldots, N\}^{d+1}}\sum_{\kappa +|\alpha| \leq m-1}
\|U\|_{W_{m, p}^{m, p}\left(\Omega_{\mu, N}\right)}
\left\|\phi_{\mu}(t,x)t^\kappa x^{\alpha}-R_{\rho_3}
\left(\Psi_{\epsilon,\Phi_{\kappa, \alpha}^\mu }\right)(t, x)\right\|_{W_{k, \infty}^{n, \infty}
\left(\Omega_{\widetilde{\mu}, N} \cap(0,1)^{d+1}\right)}
\\
\leq c_{4}N
\sum_{\substack{\mu \in\{0, \ldots, N\}^{d+1}\\ |\mu-\widetilde{\mu}|_{\ell \infty} \leq 1}}
\sum_{\kappa +|\alpha| \leq m-1}
\|U\|_{W_{m, p}^{m, p}\left(\Omega_{\mu, N}\right)}
\left\|\phi_{\mu}(t,x)t^\kappa x^{\alpha}-R_{\rho_3}
\left(\Psi_{\epsilon,\Phi_{\kappa, \alpha}^\mu }\right)(t, x)\right\|_{W_{k, \infty}^{n, \infty}
\left(\Omega_{\widetilde{\mu}, N} \cap(0,1)^{d+1}\right)}
\\
\leq c_{4} c^{\prime} N^{n+ k+1} \epsilon 
\sum_{\substack{\mu \in\{0, \ldots, N\}^{d+1}\\ |\mu-\widetilde{\mu}|_{\ell \infty} \leq 1}}
\sum_{\kappa +|\alpha| \leq m-1} \|U\|_{W_{m, p}^{m, p}\left(\Omega_{\mu, N}\right)}
\numberthis\label{eq:estimate_inner_block}
\end{multline*}
where, $c_{4}=c_{4}(m, d+1, p)>0$ is a constant,  Lemma~\ref{lemma:partition_of_unity}  and
\eqref{eq:support_mon_approx} give the second step, while \eqref{eq:approx_mon_approx}
conclude the last step.
Now, using that $\left|\left\{\kappa \in \mathbb{N}_{0},
\alpha: \alpha \in \mathbb{N}_{0}^{d},
\kappa +|\alpha| \leq m-1\right\}\right| \leq(m-1)^{d+1}$
shows that
\begin{align*}
\sum_{\substack{\mu \in\{0, \ldots, N\}^{d+1}\\ |\mu-\widetilde{\mu}|_{\ell \infty} \leq 1}}
\sum_{\kappa +|\alpha| \leq m-1}\|U\|_{W_{m, p}^{m, p}\left(\Omega_{\mu, N}\right)}
&\leq(m-1)^{d+1}
\sum_{\substack{\mu \in\{0, \ldots, N\}^{d+1}\\ |\mu-\widetilde{\mu}|_{\ell \infty} \leq 1}}
\|U\|_{W_{m, p}^{m, p}\left(\Omega_{\mu, N}\right)}
\numberthis\label{eq:estimate_sum_with_hoelder}.
\end{align*}
Combining ~\eqref{eq:estimate_inner_block} with \eqref{eq:estimate_sum_with_hoelder}
and plugging the result in ~\eqref{eq:estimate_outer_sum} finally yields
\begin{align*}
\Big\|\sum_{\mu \in\{0, \ldots, N\}^{d+1}} &\phi_{\mu}(t, x) p_{u, \mu}(t, x)-
R_{\rho_3}\left(\Phi_{P, \epsilon}\right)(t, x)\Big\|_
{W_{k, q}^{n, p}\left((0,1),(0,1)^{d}\right)}
\\
& \leq c_{4} c^{\prime}(m-1)^{d+1} N^{n+ k+ 1} \epsilon
\sum_{\widetilde{\mu} \in\{0, \ldots, N\}^{d+1}}
\sum_{\substack{\mu \in\{0, \ldots, N\}^{d+1} \\  |\mu-\widetilde{\mu}|_{\ell \infty} \leq 1}}
\|U\|_{W_{m, p}^{m, p}\left(\Omega_{\mu, N}\right)}
\\
& \leq c_{5} N^{n+ k+ 1} \epsilon\|u\|_{W_{m,p}^{m, p}\left((0,1),(0,1)^{d}\right)},
\end{align*}
where the last step is the same as Step 3 of the proof of
Lemma \ref{lemma:polynomial_approximation}
and $c_{5}>0$ depends only on $m, d+1, p, n$ and $k$.
\end{proof}

\par

Theorem \ref{thm:main} is the main result in our paper, here we show that 
any function in $ \mathcal{U}_{m,p,m ,p,d, B}$ can be approximated by neural network
with ReCU activation function
\par

\begin{theorem}\label{thm:main}
Let $d, m \in \mathbb{N}$, $n, k\in \{0,1\}$ such that $m \geq n+k+ 1$,
$1 \leq p, q \leq \infty, B>0$. Then
there exists a constant $c=c(m, d+1, p, B, n, k)>0$ with the following properties:
For any $\epsilon \in(0,1 / 2),$ there is a neural network architecture
$\mathcal{A}_{\epsilon}=\mathcal{A}_{\epsilon}(d+1, m, p, B, \epsilon)$
with $d+1$-dimensional input and one-dimensional output
such that for any $u \in \mathcal{U}_{m,p,m ,p,d, B}$ (defined in \eqref{eq:U}),
there is a neural network $\Phi_{\epsilon}^{u}$ that has architecture
$\mathcal{A}_{\epsilon}$ such that
\begin{enumerate}[(i)]
\item $L\left(\mathcal{A}_{\epsilon}\right) \leq c$;
\item $M\left(\mathcal{A}_{\epsilon}\right) \leq c \cdot \epsilon^{-\frac{d+1}{m- n-k}}$;
\item $N\left(\mathcal{A}_{\epsilon}\right) \leq c \cdot \epsilon^{-\frac{d+1}{m- n-k}}$;
\end{enumerate}
and 
\[
\left\|u - R_{\rho_3}\left(\Phi_{\epsilon}^{u}\right)\right\|_{W_{k, q}^{n, p}\left((0,1 ),(0,1)^{d}\right)}
\leq
\sqrt{\epsilon}.
\]
\end{theorem}

\par

\begin{proof}[Proof of Theorem \ref{thm:main}]
The idea of the proof is simple, indeed, we need to approximate the function $u$
by a sum of localized polynomials and then approximate the sum by a neural network.
We start by setting
\begin{equation}\label{eq:big_N}
N:=\left\lceil\left(\frac{\varepsilon'}{2 C B}\right)^{-1 /(m- n-k)}\right\rceil
\end{equation}
where $C=C(d+1, m, p)>0$ is the constant from Lemma~\ref{lemma:polynomial_approximation}.
Without loss of generality, we may assume that $C B \geq 1$.
Moreover Lemma~\ref{lemma:polynomial_approximation} yields that if
$\Psi = \left\{\phi_{\mu}: \mu \in\{0, \ldots, N\}^{d+1}\right\}$
is the partition of unity from Lemma 
\ref{lemma:partition_of_unity}, then there exist polynomials
 $p_{u,\mu}$
 where
$$
 p_{u,\mu}(t, x)=\sum_{\kappa +|\alpha| \leq m-1}
c_{\mu,\kappa, \alpha}t ^\kappa x^{\alpha} \text{ for }\mu \in\{0, \ldots, N\}^{d+1}
$$
such that
\begin{align*}
\left\|u-\sum_{\mu \in\{0, \ldots, N\}^{d+1}} \phi_{\mu} p_{u, \mu}\right\|_
{W_{k, q}^{n, p}\left( (0,1), (0,1)^{d}\right)}
& \leq C B\left(\frac{1}{N}\right)^{m- n-k}
\\
& \leq C B \frac{\varepsilon'}{2 C B}=\frac{\varepsilon'}{2}.\numberthis\label{eq:final_triangle_1}
\end{align*}

For the second step, let $C_{1}=C_{1}(m, d+1, p, n,k)>0, C_{2}=C_{2}(m, d+1)>0$
and $C_{3}=$ $C_{3}(m, d+1)>0$ be the constants from
Lemma \ref{lemma:network_polynomial_approximation}
and $\Phi_{P, \varepsilon'}$ be the neural network given in the same lemma
(independent of the function $u$) with
$\frac{\varepsilon'}{2C_1B}\left(\left(\frac{\varepsilon'}{2 C B}\right)^{-1 /(m- n-k)} +1\right)^{-n-k-1}$
instead of $\epsilon$ in \eqref{eq:loc_pol_estim}.
The neural network $\Phi_{P, \varepsilon'}$ has  $d+1$-dimensional input and
one-dimensional output, at most $C_{2}$ layers and
$$
C_3(N+1)^{2(d+1)}\leq 
C_{3}(\left(\frac{\varepsilon'}{2 C B}\right)^{-1 /(m- n-k)} +2)^{2(d+1)}
\leq
C_3 3^{2(d+1)}\left(\frac{\varepsilon'}{2 C B}\right)^{-\frac{d+1}{m- n-k}}
\leq C'' \varepsilon'^{-\frac{2(d+1)}{m- n-k}}
$$
nonzero weights and neurons, such that
$C''=C''(m, d+1, p, B, n, k)$
is a positive constant, where in the first inequality we used the fact that
$\frac{2 C B} {\varepsilon'} \geq 1$. Thus, for the statement in the theorem,
we choose $c = max(C_2, C'')$.
Furthermore, we have
\begin{align*}
&\left\|\sum_{\mu \in\{0, \ldots, N\}^{d+1}} \phi_{\mu} p_{u,\mu} -
R_{\rho_3}\left(\Phi_{P, \varepsilon'}\right)\right\|_{W_{k,q}^{n, p}\left((0,1),(0,1)^{d}\right)}
\\
& \qquad\qquad\leq C_{1}B N^{n+ k+ 1} \frac{\varepsilon'}{2C_1B}
\left(\left(\frac{\varepsilon'}{2 C B}\right)^{-1 /(m- n-k)} +1\right)^{-n-k-1}
\\
& \qquad\qquad\leq \frac {\varepsilon'} 2. \numberthis\label{eq:final_triangle_2}
\end{align*}
Using the triangle inequality and Eqs. \eqref{eq:final_triangle_1} and \eqref{eq:final_triangle_2},
we finally obtain
\[
\begin{aligned}
\left\|u -R_{\rho_3}\left(\Phi_{P, \varepsilon'}\right)
\right\|_{W_{k, q}^{n, p} \left((0,1),(0,1)^{d}\right)}
&\leq \left\|u-\sum_{\mu \in\{0, \ldots, N\}^{d+1}} \phi_{\mu} p_{u, \mu}\right\|_
{W_{k, q}^{n, p}\left( (0,1), (0,1)^{d}\right)}
\\
& +\left\|\sum_{\mu \in\{0, \ldots, N\}^{d+1}} \phi_{\mu} p_{u,\mu} -
R_{\rho_3}\left(\Phi_{P, \varepsilon'}\right)\right\|_{W_{k,q}^{n, p}\left((0,1),(0,1)^{d}\right)}
\\
\leq & \frac{\varepsilon'}{2}+\frac{\varepsilon'}{2}=\varepsilon'
\end{aligned}
\]
which concludes the proof for $\varepsilon' = \sqrt{\epsilon}$.
\end{proof}


\par

\bibliography{literature}
\bibliographystyle{abbrv}

%
%

\appendix

\section{Proof of the results in Section \ref{sec:prelim}}

\subsection{Proof of Proposition \ref{prop:remainder_estimate}}\label{app:prop_remainder_estimate}
The $m^{th}-$order remainder term is given by
$R^m u(t,  x) = u(t, x) - Q^m u(t, x)$.
\begin{align*}
u(t, x)&=\sum_{|\alpha|+k<m} \frac{1}{\alpha !k!}
D_x^{\alpha}D_t^k u(\tau, \xi)(x-\xi)^{\alpha}(t- \tau)^k 
\\
&+\sum_{|\alpha|+k=m}(x-\xi)^{\alpha}(t- \tau)^k  \int_{0}^{1}
\frac{m}{\alpha !k!} s^{m-1} D_x^{\alpha} D_t^k u(t+s(\tau-t), x+s(\xi-x)) d s.
\end{align*}
Using the previous equality and the properties of a cut-off function, we get
\begin{equation}\label{eq:RemainderDef}
\begin{aligned}
R^{m} u(t, x)=& \int_{\mathrm{B}} u(t, x) \phi(\tau, \xi) d y
-\int_{\mathrm{B}} T_{\tau, \xi}^{m} u(t, x) \phi(\tau, \xi)\, d\xi d\tau
\\
=& \int_{\mathrm{B}}\left[u(t, x)-T_{\tau, \xi}^{m} u(t, x)\right] \phi(\tau, \xi)\, d\xi d\tau
\\
=& \int_{\mathrm{B}} \phi(\tau, \xi) m\left(\sum_{|\alpha|=m}(x-\xi)^{\alpha}(t-\tau)^k \right.
\\
&\left.\times \int_{0}^{1} \frac{s^{m-1}}{\alpha !k!}
D_x^{\alpha}D _t^k u(t+s(\tau-t),x+s(\xi-x)) d s\right) d\xi d\tau.
\end{aligned}
\end{equation}

We make a change of variables from the $(\tau, \xi, s)$-space to the
$(T,\Xi, s)$-space,
where  $\Xi = x + s(\xi - x),\,  T = t + s(\tau -t)$.
Then, we have
$$
ds\, d\xi\, d\tau = s^{-(d+1)}ds\, d\Xi\, dT.
$$

The domain of integration in the $(\tau, \xi, s)$-space is $\mathrm{B} \times (0, 1]$ and the corresponding
domain in the $(T,\Xi, s)$-space is the set
$$
A=\left\{(T,\Xi, s): s \in(0,1], \;\left|\frac 1s(\Xi-x)+x-x_{0}\right|+
 \left|\frac 1s(T-t)+t-t_{0} \right|<r\right\}.
 $$
Therefore, 
$$
(T, \Xi, s) \in A \text { implies that } \frac{|\Xi-x|+ |T- t|}{\left|x-x_{0}\right|+ |t- t_0|+r}<s.
$$
Moreover, for $|\alpha|+ k = m$, we have 
\begin{equation}\label{eq:ProdPower}
(x-\xi)^{\alpha}(t-\tau)^k=s^{-m}(x-\Xi)^{\alpha}(t-T)^k.
\end{equation}
Letting $\chi_{A}$ be the characteristic function of $A,$ from \eqref{eq:RemainderDef}  and
\eqref{eq:ProdPower} we obtain
\begin{multline*}
R^{m} u(t, x)=\sum_{|\alpha|+k=m} \iint \chi_{A}(T, \Xi, s)
\phi\left(t+ \frac{T-t}{s}, x+\frac{(\Xi-x)}{s}\right)
\\
\times \frac{m}{\alpha !k!} s^{-d-2}(x-\Xi)^{\alpha}(t-T)^k
D_x^{\alpha}D _t^k u(T, \Xi) d s d \Xi dT.
\end{multline*}
The projection of $A$ onto the $(T, \Xi)$-space is $C_{t, x}$.
Therefore, by Fubini's Theorem,
\begin{align*}
R^{m} u(t, x)&=m \sum_{|\alpha|+ k=m} \int_{C_{t, x}} \frac{1}{\alpha !k!}
D_x^{\alpha} D_t ^k u(T, \Xi)( x-\Xi)^{\alpha} (t-T)^k
\\
&\qquad\times\left[\int_{0}^{1} \phi(t+\frac 1s(T-t), x+\frac 1s (z-x))
\chi _{A}(T, \Xi, s) s^{-d-2} d s\right] d \Xi dT
\\
&=m \sum_{|\alpha|+ k=m} \int_{C_{t, x}} K_{\alpha, k}(t, T; x, \Xi) 
D_x^{\alpha}D_t^k u(T, \Xi) d \Xi dT
\end{align*}
if we define 
\begin{align*}
K(t,T; x, \Xi)&=\int_{0}^{1} \phi(t+\frac 1s(T-t), x+\frac 1s (z-x)) \chi _{A}(T, \Xi, s) s^{-d-2} d s
\intertext{ and}
K_{\alpha, k}(t,T; x, \Xi))&=\frac 1{\alpha !k!}(x-\Xi)^{\alpha} (t-T)^k K(t,T; x, \Xi).
\end{align*}
It remains to prove estimate \eqref{eq:kernelEstim}  for $K(t,T; x, \Xi)$.
Thus, let $\displaystyle{y = \frac{|\Xi-x|+ |T- t|}{\left|x-x_{0}\right|+ |t- t_0|+r}}$.

Then,
\begin{align*}
|K(t,T; x, \Xi)|&=\left|\int_{0}^{1} \chi_{A}(T, \Xi, s)
	\phi(t+ \frac 1s (T-t), x+\frac 1s(\Xi-x)) s^{-d-2} d s\right|
\\
&\leq \int_{y}^{1}|\phi(t+ \frac 1s (T-t), x+\frac 1s(\Xi-x))| s^{-d-2} d s
\\
&\leq\left.\|\phi\|_{L_t^{\infty}L_x^{\infty}(\mathrm{B})}
	\frac{s^{-d-1}}{d+1}\right|_{1} ^{y}
		\leq \frac 1{d+1}\|\phi\|_{L _t^{\infty}L_x^{\infty}(\mathrm{B})} y^{-d-1}
\\
&=\frac 1{d+1}\|\phi\|_{L_t^{\infty}L _x^\infty(\mathrm{B})}
			\left(\left|x-x_{0}\right|+ |t- t_0|+r\right)^{d+1}
				\left(|\Xi-x|+ |T- t| \right)^{-d-1}
\\[1ex]
&\leq C r^{-d-1}\left(\left|x-x_{0}\right|+ |t- t_0|+r\right)^{d+1}
	\left(|\Xi-x|+ |T- t| \right)^{-d-1}
\\[1ex]
&=C\left(1 + \left(\left|x-x_{0}\right|+ |t- t_0|\right)/r\right)^{d+1}
	\left(|\Xi-x|+ |T- t| \right)^{-d-1}.\qedhere
\end{align*}

\subsection{Proof of Lemma \ref{lemma:bramble_hilbert}}\label{app:lemma_bramble_hilbert}
We need the next result for the proof of Lemma \ref{lemma:bramble_hilbert}.

\begin{lemma}\label{lem:mixLebesgueEstim}
Let $I\subset\mathbb{R}$, $\Omega\subset \mathbb{R}^d$ be open and bounded,
$k, n\in \mathbb{N}_0$ and  $m \in \mathbb{N}$, such that  $n+k = m$,
$1\leq p,q\leq \infty$,  $u \in L_t^p L _x^p(I\times\Omega)$,
and let
\begin{equation*}
g(t, x)= \int _{I\times\Omega} 
\frac{|\xi-x|^{n} |\tau-t|^k}{\left(|\xi-x|+ |\tau- t| \right)^{d+1}}
u(\tau, \xi) \, d\xi d\tau.
\end{equation*}
Then,
\begin{equation}\label{eq:ResEstim}
\left\Vert g\right\Vert _{L _t^q L _x^p} 
\leq 
 C_{m, n,k }\, h^{m(p+q)}
\left\Vert u\right\Vert _{L _t^p L _x^p}, \text{ where }h = diam(I\times\Omega).
\end{equation}
\end{lemma}

\par

\begin{proof}
First we assume that  $1<p,q<\infty$. Then,
using H\"older's inequality with $\frac 1p + \frac 1{p'}=1$, we get
\begin{multline*}
\|g\|_{L _t^q L _x^p(I\times\Omega)}^q=
\int_{I}\left(\int_{\Omega}  \left\vert  \int _{I\times\Omega} 
\frac{|\xi-x|^{n} |\tau-t|^k}{\left(|\xi-x|+ |\tau- t| \right)^{d+1}}
u(\tau, \xi) \, d\xi d\tau\right\vert ^{p}d x \right)^{q/p} dt
\\
\leq \int_{I}\left(\int_{\Omega}  \left[\left(  \int _{I\times\Omega} 
\frac{|\xi-x|^{n} |\tau-t|^{k}}{\left(|\xi-x|+ |\tau- t| \right)^{d+1}}
|u(\tau, \xi)|^p  \, d\xi d\tau \right)^\frac{1}{p}\right.\right.
\\
\left.\left.\times \left(\int _{I\times\Omega}
	\frac{|\xi-x|^{n} |\tau-t|^{k}}{\left(|\xi-x|+ |\tau- t| \right)^{d+1}}
	 \, d\xi d\tau\right)^{\frac 1{p'}}\right] ^{p}d x \right)^{q/p} dt
\\
\leq C_{d,m} \,diam{(I\times \Omega)}^{mp/{p'}}
\int_{I}\left(\int_{\Omega}  \int _{I\times\Omega} 
\frac{|\xi-x|^{n} |\tau-t|^{k}}{\left(|\xi-x|+ |\tau- t| \right)^{d+1}}
|u(\tau, \xi)|^p  \, d\xi d\tau  d x \right)^{q/p} dt
\\
\leq C_{d,m} \,diam{(I\times \Omega)}^{m(1+p/{p'})}
\int_{I}\left(  \int _{I\times\Omega} 
|\tau-t|^{k} |u(\tau, \xi)|^p  \, d\xi d\tau  \right)^{q/p} dt
\\
\leq C_{d,m} \,diam{(I\times \Omega)}^{m(1+p/{p'}+  q/p)}
\left\Vert u\right\Vert _{L _t^{p} L_x^p}^{ q}
= C_{d, m} \,diam{(I\times \Omega)}^{m(p+  q)}
\|u\|_{L _t^{p}L_x^p(I\times\Omega)}^{q}.
\end{multline*}
The cases where $p, q \in \{1, \infty\}$ are straightforward
and therefore left to the reader.

\end{proof}

\par

\begin{proof}[Proof of Lemma \ref{lemma:bramble_hilbert}]
In a similar way as in \cite[Lemma 4.3.8]{brenner2007mathematical}, we prove
Lemma \ref{lemma:bramble_hilbert} using \cite[Proposition 4.1.9]{brenner2007mathematical}
and the fact that $Q^m$ is a polynomial
in both $t$ and $x$ of order less than $m$.
Let $ diam (I\times\Omega) = 1$, for $k=n=0$, using Lemma \ref{lem:mixLebesgueEstim}, 
we get 
\begin{multline*}
\Vert  u-  Q^mu\Vert_{L _t^q L _x^p(I\times\Omega)}
=
\Vert  R^mu\Vert_{L _t^q L _x^p(I\times\Omega)}
\\
\leq
m \sum _{|\alpha|+ \kappa=m} \left\Vert\int _{I\times\Omega} K_{ \alpha, \kappa}(t, T; x, \Xi)
 D_x^\alpha D_t^\kappa u(\tau, \xi) \, d\Xi dT\right\Vert_{L _t^q L _x^p(I\times\Omega)}
 \\
\leq C_{m, d}(1+r^{-1})^{d+1}\sum _{|\alpha|+ \kappa=m} \left\Vert\int _{I\times\Omega} 
\frac{|\Xi-x|^{|\alpha|} |T-t|^\kappa}{\left(|\Xi-x|+ |T- t| \right)^{d+1}}
	D_x^\alpha D_t^\kappa u(\tau, \xi) \, d\Xi dT\right\Vert_{L _t^q L _x^p(I\times\Omega)}
\\
\leq C_{d, m, r}  \left\Vert u\right \Vert_ {W_{m,p}^{m,p}(I,\Omega)}.
\end{multline*}

For $0<k+n\leq m$, 

$$
\begin{aligned}
\left|u-Q^{m} u\right|_{W_{k,q}^{n,p}(I, \Omega)}
&=\left|R^{m} u\right|_{W_{k,q}^{n,p}(I, \Omega)}
\leq \sum_{|\alpha|+ \kappa=k +n}\left\Vert R^{m-k-n} D_x^{\alpha}
D _t^\kappa u\right\Vert_{L _t^q L_x^p}
\\
&\leq C_{d,m,r}\sum_{|\alpha|+ \kappa=k +n}\left\Vert
D_x^{\alpha}D_t^\kappa u\right\Vert_{W_{m-k-n, p}^{m-k-n,p}(I,\Omega)}
\\
&\leq C_{d,m,, r}\Vert u\Vert _{W_{m, p}^{m,p}(I, \Omega)}.
\end{aligned} 
$$

For a general domain $\Omega$, we define 
$\Theta =  \{x/h \text{ such that } x \in \Omega\}$,
using similar argument to the previous calculus,
we conclude the result. The details are left to the reader.
\end{proof}

\par


\subsection{Proof of Lemma \ref{prop:taylor_is_polynom}}\label{app:prop_taylor_is_polynom}
The first part of the proof of this lemma follows closely the chain of arguments
in \cite[Equations~(4.1.5) - (4.1.8)]{brenner2007mathematical} and the Binomial theorem.
We write for $\kappa\in \mathbb{N}_0$ and $\alpha\in\mathbb{N}_0^d$ 
\[
(t-\tau)^\kappa(x-y)^\alpha=
\sum_{\nu +\mu=\kappa}\sum_{\substack{\gamma,\beta \in \mathbb{N}_0^d, \vspace{0.2em}
\\ \gamma+\beta=\alpha}}
a_{(\nu,\mu,\gamma,\beta)}t^\nu \tau^{\mu} x^\gamma y^\beta,
\]
where $a_{(\nu,\mu, \gamma,\beta)}\in\mathbb{R}$ are suitable constants with 
\begin{equation}\label{eq:sob_multinomial}
|{a_{(\nu,\mu,\gamma,\beta)}}|\leq \binom{\nu+\mu}{\nu}\binom{\gamma+\beta}{\gamma}
=\frac{(\nu+\mu)!(\gamma+\beta)!}{\nu!\mu!\gamma\,!\;\beta\,!}
\end{equation}
in multi-index notation. Then, combining Equation \eqref{eq:sob_taulor_sub}
and \eqref{eq:sob_taylor} yields
\begin{multline*}
Q^{k+n} u(x) = \sum_{|{\alpha}|+ \kappa  \leq k+ n-1}
\sum_{\substack{\gamma +\beta=\alpha\vspace{0.2em}\\ \nu +\mu = \kappa } }
\frac{1}{\kappa!\alpha!}a_{(\nu,\mu,\gamma,\beta)}t^\nu x^\gamma
\int_\mathrm{B} D_x^\alpha D _t ^\kappa u(\tau, y)\tau^{\mu}y^\beta\phi(\tau, y)dy d\tau
\\
=\sum_{|{\gamma}|+ \nu \leq k+ n-1}t^\nu x^\gamma
\underbrace{\sum_{|{\gamma+\beta}|+ \nu+\mu\leq k+ n-1}\frac{1}{(\gamma+\beta)!(\nu+\mu)!}
a_{( \nu, \mu,\gamma,\beta)} \int_\mathrm{B} D_x^{\gamma+\beta}D_t^{\nu +\mu } u(\tau, y)
t^{\mu}y^\beta\phi(\tau, y)dyd\tau}_{=:c_{\gamma, \nu}}.\numberthis\label{eq:c_gamma}
\end{multline*}
For the second part, note that
\begin{align*}
\left|{ \int_\mathrm{B} D_x^{\gamma+\beta}D_t^{\nu +\mu } u(\tau, y)
t^{\mu}y^\beta\phi(\tau, y)dyd\tau}\right|&
\leq  \int_\mathrm{B} |D_x^{\gamma+\beta}D_t^{\nu +\mu } u(\tau, y)|
|t|^{\mu}|y|^{|\beta|}|\phi(\tau, y)|dyd\tau
\\
&\leq R^{|{\beta}|+ \mu} \Vert{u}\Vert_{W_{k+n -1, p}^{k+ n-1, p}(\mathrm{B})}
\Vert{\phi}\Vert_{L_ t^{q}L_x^{ q}(\mathrm{B})}\numberthis\label{eq:sob_taylor_coeff},
\end{align*}
where we used the fact that $\mathrm{B}\subset B_{R,\Vert{\cdot}\Vert_{\ell^\infty}}(0)$ and
the H{\"o}lder's inequality with $1/p=1-1/{q}$.
Next, since
$\phi\in L _t^1L_x^1(\mathrm{B})\cap L_t^\infty L_x^\infty(\mathrm{B})$ and $\Vert{\phi}\Vert_{L_t^1L _x^1}=1$,
using the Mixed interpolative H\"older’s inequality cf.\,\cite{Grey}, we get
\[
\Vert{\phi}\Vert_{L_ t^{q}L_x^{ q}}\leq \Vert{\phi}\Vert_{L_t^1L _x^1}^{1/q}
\Vert{\phi}\Vert_{L_t^\infty L _x^\infty}^{1-1/q}=
\Vert{\phi}\Vert_{L_t^\infty L _x^\infty}^{1/p}.
\]
Combining the last estimate with Equation~\eqref{eq:sob_taylor_coeff} yields
\begin{align*}
\left|{ \int_\mathrm{B} D_x^{\gamma+\beta}D_t^{\nu +\mu } u(\tau, y)
t^{\mu}y^\beta\phi(\tau, y)dyd\tau}\right|&\leq
 R^{k+n-1} \Vert{u}\Vert_{W_{k+n -1, p}^{k+ n-1, p}(\Omega)}
 \Vert{\phi}\Vert_{L_t^\infty L _x^\infty(\mathrm{B})}^{1/p}
\\
&\leq c R^{k+ n-1}r^{-(d+1)/p}
\Vert{u}\Vert_{W_{k+n -1, p}^{k+ n-1, p}( I,\Omega)}\numberthis\label{eq:sob_int},
\end{align*}
where the second step follows from $\Vert{\phi}\Vert_{L^\infty}\leq c r^{-(d+1)}$
for some constant $c=c(d)>0$ (see~\cite[Section~4.1]{brenner2007mathematical}).
To estimate the absolute value of the coefficients $c_{\gamma, \nu}$
(defined in Equation~\ref{eq:c_gamma}), we have
\begin{align*}
|{c_{\gamma, \nu}}|&\leq \sum_{|{\gamma+\beta}|+ \nu+\mu\leq k+ n-1}
\frac{1}{(\gamma+\beta)!(\nu+\mu)!} \left|a_{(\nu,\mu, \gamma,\beta)} \right|
\left|\int_\mathrm{B} D_x^{\gamma+\beta}D_t^{\nu +\mu } u(\tau, y)
t^{\mu}y^\beta\phi(\tau, y)dyd\tau\right|
\\
&\leq  c R^{k+ n-1}r^{-(d+1)/p} \Vert{u}\Vert_{W_{k+n -1, p}^{k+ n-1, p}( I,\Omega)}
\sum_{|{\gamma+\beta}|+ \nu+\mu\leq k+ n-1}
\frac{1}{\gamma!\beta!\nu!\mu!}
\\
&=c'r^{-(d+1)/p} \Vert{u}\Vert_{W_{k+n -1, p}^{k+ n-1, p}( I,\Omega)},
\end{align*}
where $c'=c'(k, n,d,R)>0$ is a constant.

\par


\subsection{Proof of Lemma\ref{lemma:product_rule_bound_p}}\label{app:lemma_product_rule_bound_p}
From the given assumptions on $f$ and $g$ it is clear that $fg\in L_t^qL_x^p(I\times\Omega)$.
Moreover, 
$fg, (D_t f)g + f(D_t g), (D_{x_i} f)g+ f (D _{x_i} g)
(D_{x_i}D_ t f)g + (D _tf) (D_{x_i}g) + (D_{x_i}f)(D_t g)+ f(D_{x_i}D_ t g)
\in L^1_{\text{loc}}(I\times\Omega)$ so that
the product formula in \cite[Chapter 7.3]{gilbarg1998elliptic} yields
that for the weak derivatives of $fg$ it holds 
\begin{align*}
D_ t (fg)&= (D_t f)g+ f (D _t g), 
\quad
D_ {x_i} (fg)= (D_{x_i} f)g+ f (D _{x_i} g)
\intertext{and}
D_{x_i}D_ t (fg)& = (D_{x_i}D_ t f)g + (D _tf) (D_{x_i}g)
	+ (D_{x_i}f)(D_t g)+ f(D_{x_i}D_ t g)
\end{align*}
for $i=1,2,\ldots,d$. Thus, we have
\begin{align*}
|{fg}|_{W_{0,q}^{1,  p}(I,\Omega)}&=|{fg}|_{L _t^q(I,  W_x^{1,  p}(\Omega))}
=\left\Vert\left( \sum_{i=1}^d
\Vert{(D_{x_i} f)g+ f (D_{x_i} g)}\Vert_{L_x^p(\Omega)}^p\right)^{1/p}\right\Vert_{L_t^q(I)}
 \\
&\leq \left\Vert\sum_{i=1}^d \Vert{(D_{x_i} f)g+ f (D_{x_i} g)}\Vert_{L _x^p(\Omega)}\right\Vert_{L_t^q(I)}
\\
&\leq \sum_{i=1}^d \Vert{(D_{x_i} f)g}\Vert_{L_t^qL_x^p(I\times\Omega)}
+ \Vert{f (D_{x_i} g)}\Vert_{L_t^qL_x^p(I\times\Omega)},
\end{align*}
where
\begin{align*}
\Vert{(D_{x_i} f)g}\Vert_{L_t^q L_x^p(I\times\Omega)}
	&\leq \Vert D_{x_i} f\Vert_{L _t^\infty L_x^\infty(I\times\Omega)}\Vert{g}
	\Vert_{L_t^q L_x^p(I\times\Omega)}
	\leq |f|_{W_{0, \infty}^{1,\infty}(I, \Omega)}\Vert{g}\Vert_{L_t^q L_x^p(I\times\Omega)},
\\[1ex]
	\Vert{f (D_{x_i} g)}\Vert_{L_t^qL_x^p(I\times\Omega)}
	&\leq \Vert{f}\Vert_{L _t^\infty L_x^\infty(I\times\Omega)}
	\Vert{D_{x_i} g}\Vert_{L_t^qL_x^p(I\times\Omega)}.
\end{align*}
Thus,

\begin{align*}
|{fg}|_{W_{0,q}^{1,  p}(I,\Omega)}
&\leq 
d |f|_{W_{0, \infty}^{1,\infty}(I,\Omega)}\Vert{g}\Vert_{L_t^q L_x^p(I\times\Omega)} 
+ \Vert{f}\Vert_{L _t^\infty L_x^\infty(I\times\Omega)}
\sum _{i=1}^d\Vert{D_{x_i} g}\Vert_{L_t^qL_x^p(I\times\Omega)}
\\[1ex]
&\leq C \left(|f|_{W_{0, \infty}^{1,\infty}(I\times\Omega)}
\Vert{g}\Vert_{L_t^q L_x^p(I\times\Omega)}
+ \Vert{f}\Vert_{L _t^\infty L_x^\infty(I\times\Omega)}
\vert{g}\vert_{W_{0,q}^{1, p}(I\times\Omega)}\right),
\end{align*}
where $C>0$ depends on $d$ and $p$.

Moreover,
\begin{align*}
|{fg}|_{W_{1,q}^{0,  p}((I, \Omega))}&=|{fg}|_{W_t^{1,  q}(I, L _x^p (\Omega))}
=\left\Vert{(D_t f)g+ f (D _t g)}\right\Vert_{L_t^qL_x^p(I\times\Omega)}
 \\
&\leq \left\Vert{(D_t f)g}\right\Vert_{L_t^qL_x^p(I\times\Omega)}+
\left\Vert{f (D _t g)}\right\Vert_{L_t^qL_x^p(I\times\Omega)},
\end{align*}
such that
\begin{align*}
\left\Vert{(D_t f)g}\right\Vert_{L_t^qL_x^p(I\times\Omega)}
&\leq \left\Vert{D_t f}\right\Vert_{L_t^\infty L_x^\infty(I\times\Omega)}
 \left\Vert{g}\right\Vert_{L_t^qL_x^p(I\times\Omega)}
\leq  \left|{f}\right|_{W_{1,\infty}^{0, \infty}(I, \Omega)}
 \left\Vert{g}\right\Vert_{L_t^qL_x^p(I\times\Omega)}
\\[1ex]
\left\Vert{f (D _t g)}\right\Vert_{L_t^qL_x^p(I\times\Omega)}
&\leq \left\Vert{f }\right\Vert_{L_t^\infty L_x^\infty(I, \Omega)}
	\left\Vert{D _t g}\right\Vert_{L_t^qL_x^p(I\times\Omega)}
		\leq \left\Vert{f }\right\Vert_{L_t^\infty L_x^\infty(I\times\Omega)}
		\left|{g}\right|_{W_{1, q}^{0,p}(I, \Omega)}.
\end{align*}
Then, we get
\[
|{fg}|_{W_{1,q}^{0,  p}((I, \Omega))} \leq
\left|{f}\right|_{W_{1,\infty}^{0, \infty}(I, \Omega)}
 \left\Vert{g}\right\Vert_{L_t^qL_x^p(I\times\Omega)}
+\left\Vert{f }\right\Vert_{L_t^\infty L_x^\infty(I\times\Omega)}
		\left|{g}\right|_{W_{1, q}^{0,p}(I, \Omega)}.
\]
For the mixed derivatives, we have
\begin{multline*}
|{fg}|_{W_{1,q}^{1,  p}((I, \Omega))}=|{(D_t f)g+ f (D _t g)}|_{L _t^q(I,  W_x^{1,  p}(\Omega))}
\\
=\left\Vert\left( \sum_{i=1}^d
\Vert
{(D_{x_i}D_ t f)g + (D _tf) (D_{x_i}g) + (D_{x_i}f)(D_t g)+ f(D_{x_i}D_ t g)}
\Vert_{L_x^p(\Omega)}^p\right)^{1/p}\right\Vert_{L_t^q(I)}
 \\
\leq \left\Vert\sum_{i=1}^d
\Vert
{(D_{x_i}D_ t f)g + (D _tf) (D_{x_i}g) + (D_{x_i}f)(D_t g)+ f(D_{x_i}D_ t g)}
\Vert_{L _x^p(\Omega)}\right\Vert_{L_t^q(I)}
\\
\leq \sum_{i=1}^d \Vert{(D_{x_i}D_ t f)g}\Vert_{L_t^qL_x^p(I\times\Omega)}
+ \Vert{(D _tf) (D_{x_i}g)}\Vert_{L_t^qL_x^p(I\times\Omega)}
\\
\qquad+ \Vert{(D_{x_i}f)(D_t g)}\Vert_{L_t^qL_x^p(I\times\Omega)}
+ \Vert{f(D_{x_i}D_ t g)}\Vert_{L_t^qL_x^p(I\times\Omega)}.
\end{multline*}
Note that we have
\begin{align*}
\Vert{(D_{x_i}D_ t f)g}\Vert_{L_t^qL_x^p(I\times\Omega)}
&\leq \Vert{D_{x_i}D_ t f}\Vert_{L_t^\infty L_x^\infty(I\times\Omega)}
		\Vert{g}\Vert_{L_t^qL_x^p(I\times\Omega)}
		\leq |{f}|_{W_{1, \infty}^{1, \infty}(I, \Omega)}
			\Vert{g}\Vert_{L_t^qL_x^p(I\times\Omega)}
\\[1ex]
\Vert{(D _tf) (D_{x_i}g)}\Vert_{L_t^qL_x^p(I\times\Omega)}
&\leq \Vert{D _tf}\Vert_{L_t^\infty L_x^\infty(I\times\Omega)}
		\Vert{D_{x_i}g}\Vert_{L_t^qL_x^p(I\times\Omega)}
		\leq |{f}|_{W_{1, \infty}^{0, \infty}(I, \Omega)}
			\Vert{D_{x_i}g}\Vert_{L_t^qL_x^p(I\times\Omega)}
\\[1ex]
\Vert{(D_{x_i}f)(D_t g)}\Vert_{L_t^qL_x^p(I\times\Omega)}
&\leq \Vert{D_{x_i}f}\Vert_{L_t^\infty L_x^\infty(I\times\Omega)}
		\Vert{D_t g}\Vert_{L_t^qL_x^p(I\times\Omega)}
		\leq |{f}|_{W_{0, \infty}^{1, \infty}(I, \Omega)}
			|{g}|_{W_{1, q}^{0, p}(I, \Omega)}
\\[1ex]
\Vert{f(D_{x_i}D_ t g)}\Vert_{L_t^qL_x^p(I\times\Omega)}
&\leq \Vert{f}\Vert_{L_t^\infty L_x^\infty(I\times\Omega)}
		\Vert{D_{x_i}D_ t g}\Vert_{L_t^qL_x^p(I\times\Omega)}
\end{align*}
Finally, we get
\begin{multline*}
|{fg}|_{W_{1,q}^{1,  p}(I, \Omega)}\leq
C\left( |{f}|_{W_{1, \infty}^{1, \infty}}\Vert{g}\Vert_{L_t^qL_x^p(I\times\Omega)}
+|{f}|_{W_{1, \infty}^{0, \infty}(I, \Omega)}
\vert{g}\vert_{W_{0, q}^{1,p}(I, \Omega)}\right.
\\
\left. \qquad\qquad+|{f}|_{W_{0, \infty}^{1, \infty}} |{g}|_{W_{1, q}^{0, p}(I, \Omega)}
+ \Vert{f}\Vert_{L_t^\infty L_x^\infty(I\times\Omega)}
\vert{ g}\vert_{W_{1,q}^{1, p}(I, \Omega)}
\right)
\end{multline*}
where $C=C(d,p)>0$ is a constant.

\subsection{Proof of Lemma \ref{cor:composition_norm}}\label{app:cor_composition_norm}
	We start by the case $n+k=1$.
	Let $\nabla = (\nabla_1, \nabla_2)$, where $\nabla_1$ and $\nabla_2$ are the gradient with
	respect to 	the first and second block of variables respectively.
	Moreover, we set for $j=1,\ldots,m$
	\[
		L_j:=\Vert\;\vert{\nabla f_j}\vert\;\Vert_{L_t^\infty L_x^\infty(\Omega_1\times\Omega_2)}
		\quad\text{and}\quad L_f:=(L_1,\ldots,L_m).
	\]
	Using similar ideas from the proof of \cite[Corollary B.5]{GuhKutPet},
	we conclude that $f_j$ is $L_j$-Lipschitz and therefore $f$ is $\vert L_f\vert$-Lipschitz.
	Similarly, $g$ is $L_g$-Lipschitz, where
	$L_g:=\Vert\;\vert \nabla g\vert\;\Vert_{L_t^\infty L_x^\infty(\Theta_1\times\Theta_2)}$.
	Furthermore, $g\circ f\in W_{k, \infty}^{n,\infty}(\Omega_1, \Omega_2)$, and
	$
		\Vert\;\vert\nabla (g\circ f)\vert\;\Vert_{L_t^\infty L_x^\infty(\Omega_1\times\Omega_2)}\leq
		\vert{L_f}\vert\cdot L_g.
	$
	Thus, we have
	\begin{align*}
		\vert{g\circ f}\vert_{W_{k, \infty}^{n,\infty}(\Omega_1, \Omega_2)}&\leq
		\Vert\;\vert\nabla (g\circ f)\vert\; \Vert_{L_t^\infty L_x^\infty (\Omega_1\times\Omega_2)}
		\leq\vert{L_f}\vert \cdot L_g
		\\[1ex]
		&\leq  m \Vert{L_f}\Vert_{\ell^\infty}\cdot  m 
			\vert{g}\vert_{W_{k, \infty}^{n,\infty}(\Theta_1, \Theta_2)}
		\\[1ex]
		&\leq  p\,   m^2 \vert{f}\vert_{ W_{k, \infty}^{n,\infty}
					(\Omega_1,\Omega_2)}
			\cdot \vert{g}\vert_{W_{k, \infty}^{n,\infty}(\Theta_1, \Theta_2)},
	\end{align*}
	where we use the estimate of the $\ell^2$ norm with the $\ell^\infty$ norm on $\mathbb{R}^m$
	and the fact that if $f\in  W_{k, \infty}^{n,\infty}(\Omega_1, \Omega_2)$ and that
	$\dim(\Omega_i)={p_i}$ for $i=1,2$, then
	we have the following observation
	$$
		\vert{f}\vert_{ W_{k, \infty}^{n,\infty}(\Omega_1, \Omega_2)}\leq
		\Vert{\;\vert{\nabla f}\vert\;}\Vert_{L_t^\infty L_x^\infty(\Omega_1\times \Omega_2)}
		\leq \sqrt{p_1 p_2} \vert{f}\vert_{W_{k, \infty}^{n,\infty}(\Omega_1, \Omega_2)}
		\leq p \vert{f}\vert_{W_{k, \infty}^{n,\infty}(\Omega_1, \Omega_2)}.
	$$
	If $n=k=1$, we denote by $D_1$ the derivative with respect to the first
	block of variables, then for $|\alpha|=1$ we have
	$D_1^\alpha(g\circ f)= \sum _{j=1}^m \partial_i f_j \left((\partial_j g)\circ f\right)$.
	Since $\dim(\Omega_1)=p_1$, we get 
	\begin{multline*}
		\vert{g\circ f}\vert_{W_{1, \infty}^{1,\infty}(\Omega_1,\Omega_2)}\leq
		 \max_{1\leq i\leq p_1}\sum _{j=1}^m \vert\partial_i f_j  (\partial_j g)\circ f\vert_
		 {W_{0, \infty}^{1,\infty}(\Omega_1,\Omega_2)}
		 \\
		 \leq\max_{p_1+1\leq \iota \leq p} \max_{1\leq i\leq p_1}\left(
		 \sum _{j=1}^m \Vert\partial_\iota\partial_i f_j  (\partial_j g)\circ f\Vert_
		 {L_t^\infty L_x^\infty(\Omega_1\times\Omega_2)} +
		 \Vert\partial_i f_j \partial_\iota \left((\partial_j g)\circ f\right)\Vert_
		 {L_t^\infty L_x^\infty(\Omega_1\times\Omega_2)}\right)
		 \\
		 \leq \sum _{j=1}^m \vert f_j\vert_{W_{1, \infty}^{1,\infty}(\Omega_1,\Omega_2)}
		 \Vert (\partial_j g)\circ f\Vert_{L_t^\infty L_x^\infty(\Omega_1\times\Omega_2)}+
		 \vert f_j\vert_{W_{0, \infty}^{1,\infty}(\Omega_1,\Omega_2)}
		 \vert \left((\partial_j g)\circ f\right)\vert_{W_{1, \infty}^{0,\infty}
		 			(\Omega_1,\Omega_2)}
		 \\
		 \leq\sum _{j=1}^m \vert f_j\vert_{W_{1, \infty}^{1,\infty}(\Omega_1,\Omega_2)}
		 \Vert (\partial_j g)\Vert_{L_t^\infty L_x^\infty(\Theta_1\times\Theta_2)}+
		 p m^2\vert f_j\vert_{W_{0, \infty}^{1,\infty}(\Omega_1,\Omega_2)}
		 \vert{f}\vert_{ W_{0, \infty}^{1,\infty}(\Omega_1,\Omega_2)}
		\vert{\partial_j g}\vert_{W_{0, \infty}^{1,\infty}(\Theta_1, \Theta_2)}
		\\
		 \leq\sum _{j=1}^m \vert f_j\vert_{W_{1, \infty}^{1,\infty}(\Omega_1,\Omega_2)}
		 \vert g\vert_{W_{1, \infty}^{0, \infty}(\Theta_1,\Theta_2)}+
		 p m^2\vert f_j\vert_{W_{0, \infty}^{1,\infty}(\Omega_1,\Omega_2)}
		 \vert{f}\vert_{ W_{0, \infty}^{1,\infty}(\Omega_1,\Omega_2)}
		\vert{g}\vert_{W_{1, \infty}^{1,\infty}(\Theta_1, \Theta_2)}
		\\
		 \leq m \vert f\vert_{W_{1, \infty}^{1,\infty}(\Omega_1,\Omega_2)}
		 \vert g\vert_{W_{1, \infty}^{0, \infty}(\Theta_1,\Theta_2)}+
		 p m^3\vert f\vert_{W_{0, \infty}^{1,\infty}(\Omega_1,\Omega_2)}
		 \vert{f}\vert_{ W_{0, \infty}^{1,\infty}(\Omega_1,\Omega_2)}
		\vert{g}\vert_{W_{1, \infty}^{1,\infty}(\Theta_1, \Theta_2)}
		\\
		 \leq 2\max\left(m \vert f\vert_{W_{1, \infty}^{1,\infty}
		 			(\Omega_1,\Omega_2)}
		 \vert g\vert_{W_{1, \infty}^{0, \infty}(\Theta_1,\Theta_2)},
		 p m^3\vert f\vert_{W_{0, \infty}^{1,\infty}(\Omega_1,\Omega_2)}^2
		\vert{g}\vert_{W_{1, \infty}^{1,\infty}(\Theta_1, \Theta_2)}
		\right)	
	\end{multline*}

\end{document}